\documentclass{article}
\usepackage{marvosym} % for \Letter
\usepackage{iclr2027_conference}
\usepackage{iftex}
\ifPDFTeX
  \usepackage{times}
\else
  \usepackage{fontspec}
\fi
\usepackage[utf8]{inputenc}
\usepackage{amsmath,amsfonts,amssymb,amsthm}
\usepackage{microtype}
\usepackage{booktabs,multirow,array}
\usepackage{graphicx}
\usepackage{float}
\usepackage{longtable}
\usepackage{textcomp}
\usepackage[hyphens]{url}
\usepackage{hyperref}
\hypersetup{hidelinks}
\usepackage{enumitem}
\usepackage{caption}
\let\amsEqref\eqref
\usepackage{amsmath,amsfonts,bm}

\def\eqref#1{equation~\ref{#1}}
\def\ceil#1{\lceil #1 \rceil}

\def\1{\bm{1}}

\DeclareMathAlphabet{\mathsfit}{\encodingdefault}{\sfdefault}{m}{sl}
\SetMathAlphabet{\mathsfit}{bold}{\encodingdefault}{\sfdefault}{bx}{n}

\newcommand{\E}{\mathbb{E}}

\newcommand{\KL}{D_{\mathrm{KL}}}

\DeclareMathOperator*{\argmax}{arg\,max}

\let\eqref\amsEqref

\newtheorem{theorem}{Theorem}
\newtheorem{lemma}{Lemma}
\newtheorem{assumption}{Assumption}
\newtheorem{corollary}{Corollary}
\newtheorem{proposition}[theorem]{Proposition}
\newcounter{algorithm}
\renewcommand{\thealgorithm}{\arabic{algorithm}}

\title{PAC-CF: Calibrating Irreversible Frontier\\
Pruning in LLM-Guided Search}

\iclrfinalcopy

\author{
Tianhao Qian$^{1}$,
Jiayu Chen$^{2,*}$,
Zhenyu Sun$^{3}$,
Lixu Wang$^{1,*}$
\\[5pt]
$^{1}$ The Chinese University of Hong Kong, Shenzhen \\
$^{2}$ The University of Hong Kong \\
$^{3}$ TikTok \\[2pt]
$^{*}$Corresponding authors
}

\date{}

\fancypagestyle{preprintpages}{%
  \fancyhf{}%
  \fancyhead[L]{\textit{Preprint}}%
  \fancyfoot[C]{\thepage}%
  \renewcommand{\headrulewidth}{0.4pt}%
  \renewcommand{\footrulewidth}{0pt}%
}
\fancypagestyle{firstpage}{%
  \fancyhf{}%
  \fancyfoot[C]{\thepage}%
  \renewcommand{\headrulewidth}{0pt}%
  \renewcommand{\footrulewidth}{0pt}%
}

\begin{document}
\raggedbottom

% Let AI TRUST logo
\noindent
\includegraphics[width=0.35\textwidth]{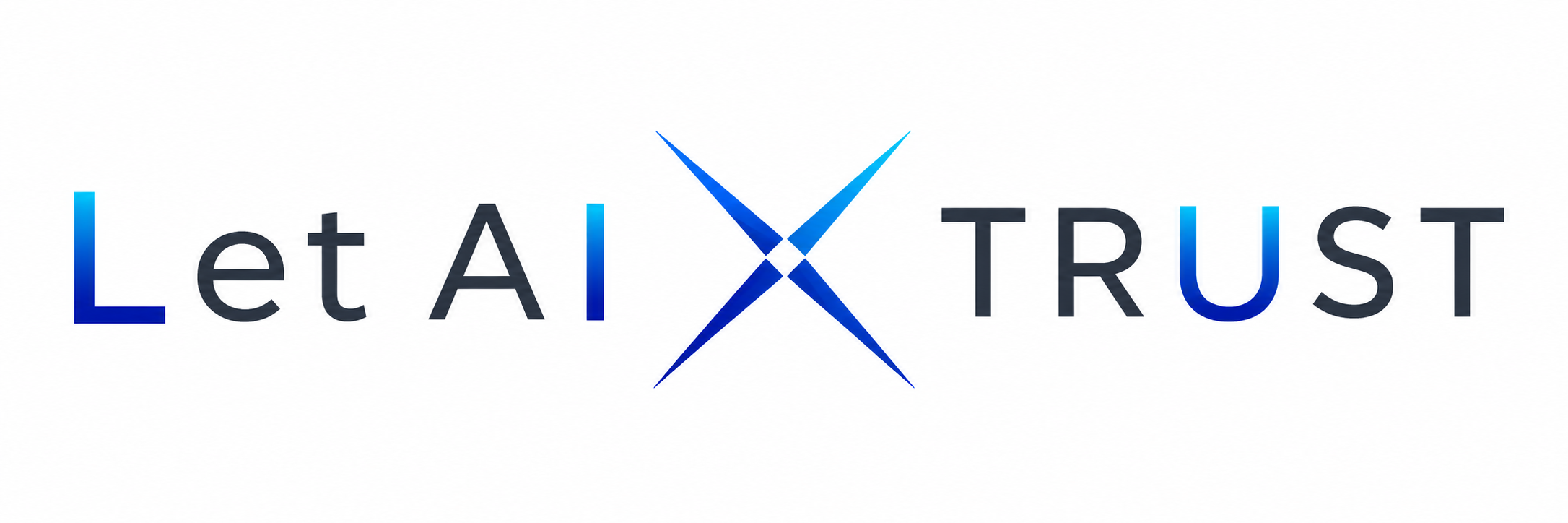}

\vspace{-25pt}
\noindent\rule{\linewidth}{0.4pt}
\vspace{0.6em}

\maketitle
\pagestyle{preprintpages}
\thispagestyle{firstpage}

\begin{abstract}

LLM-guided search explores multiple candidate trajectories, but at substantial test-time cost. Pruning low-scoring frontier candidates can control this cost, yet it also turns potentially biased evaluator scores into irreversible decisions: systematic ranking errors can persist under repeated scoring and remove useful branches. We propose Probably Approximately Correct Conformal Filtering (PAC-CF). Its fixed-frontier analysis formulates elimination as an $(\varepsilon,\delta)$-PAC problem under bounded evaluator bias; its operational rule separately calibrates a score-gap threshold on held-out tasks by running the original controller without PAC-CF and using post-search verifier labels to measure the deficit of solution-preserving candidates relative to the frontier leader. Conditional on exchangeable native-controller tasks with nonempty protected exposure, conformal calibration gives finite-sample coverage for retaining at least one verifier-defined valid continuation at every protected frontier on the native trajectory. At deployment, PAC-CF removes only candidates whose gap from the highest frontier score exceeds the frozen threshold. We evaluate PAC-CF across three domains, five controllers, and four request budgets from B100 to B500. In the cross-domain/controller macro averages, the point estimates for all three workload measures are lower at every budget; the paired-bootstrap 95\% confidence interval for utility excludes zero at B100 and B200. For pruning-aware ToolTree, the full-test-set cross-domain utility difference is $+4.38$ points at each tested budget; on the natural-termination sensitivity cohort, physical requests decrease by $18.94$--$18.95\%$ and end-to-end token usage by $23.57$--$23.76\%$.
\end{abstract}

\section{Introduction}
\label{sec:intro}

LLM-based agents increasingly decompose high-level goals into subgoals, interleave reasoning with environment interactions, and invoke external tools over multiple steps \citep{wei2022chain,wang2023selfconsistency,react2023,toolformer2023,toolllm2024,jiao2025can,li2026webcloak,wang2026person,feng2025token}. Tree search is commonly used to explore alternative action trajectories \citep{yu2025exact}. In existing approaches, an LLM-based evaluator often scores candidate nodes, and a controller uses those scores to determine which branches to expand. Compared with single-path greedy search without backtracking, tree search maintains alternative continuations and has improved performance on several reasoning, planning, and tool-use benchmarks \citep{yao2024tree,hao2023reasoning,zhou2024lats}, at the cost of additional test-time computation.

A standard LLM-guided search generates and scores candidate next actions, and a controller retains a subset for subsequent exploration; we call the scored candidate set a \emph{frontier}. Restricting further exploration through \emph{pruning} reduces cost but may discard useful branches. Prior work studies bounded-width selection, score-based prioritization, and explicit action or state pruning~\citep{xie2023selfeval,koh2024treesearch,toolchain2024,tooltree2026,fetch2025,equivpruner2026}. When admission is driven by LLM-based scores, these methods can inherit systematic errors from miscalibration, presentation effects, or model misspecification~\citep{haldar2025rating,li2024misspecification,shi2025positionbias,thakur2024judging,park2025prmcalibration}. Averaging repeated evaluations can reduce stochastic variation, but need not remove persistent ranking error. This matters because pruning is effectively irreversible within a run: a removed candidate's subtree is not explored unless the controller later regenerates it.

A controlled study on 50-node Euclidean TSP instances~\citep{kool2019attention} illustrates that such evaluator bias can remain substantial even on relatively simple tasks. Despite averaging 16 independent evaluator scores per candidate, the globally optimal continuation ranks only 5.11 on average at conflict frontiers. Consequently, fixed-cardinality top-$4$ pruning removes the optimal continuation at $64.4\%$ of these frontiers, while top-$6$ pruning still removes it at $39.1\%$. The effect compounds at the task level: among 24 conflict instances, none preserves the optimal reference path across all inspected frontiers under top-$4$ pruning, and only 1 of 24 does so under top-$6$. These results show that repeated evaluation can reduce stochastic score variation, yet persistent ranking errors can remain and lead to irreversible pruning errors. This motivates explicitly calibrating irreversible pruning decisions against evaluator uncertainty.

PAC-CF follows a \emph{theory--calibration--deployment} pipeline. The fixed-frontier theory models bounded evaluator bias and derives a score margin sufficient for reliable elimination. Separately, calibration runs the original (\emph{native}) controller and uses post-search verifier labels to measure the deficit of solution-preserving candidates relative to the frontier leader. The resulting threshold is frozen before deployment. The deployment rule removes only candidates whose score gap exceeds that threshold and then returns the survivors to the native controller. We evaluate the same rule on Game24~\citep{yao2024tree}, Blocksworld~\citep{valmeekam2023planbench}, and GTA~\citep{wang2024gta} with five representative controllers~\citep{toolchain2024,zhou2024lats,tooltree2026}. Across budgets of 100, 200, 300, and 500 requests per task, the cross-domain/controller macro-average point estimates for all three workload metrics are lower, while the utility confidence interval excludes zero only at B100 and B200. At B500, the full-paired-set macro reductions are $9.47\%$ for physical requests, $12.40\%$ for Net Graph nodes, and $10.25\%$ for tokens. Our contributions are:
\begin{itemize}[leftmargin=*]
    \item \noindent\textbf{Theory:} We formulate irreversible frontier pruning as a PAC decision problem. The analysis models persistent bias through bounded evaluator bias, separates it from reducible sampling noise, derives the effective separation
    $\Delta-4L$ sufficient for reliable elimination, and identifies lower-bound
    regimes where repeated scoring cannot make deletion reliable.
    \item \noindent\textbf{Calibration:} We calibrate PAC-CF on trajectories generated by the native controller. Each calibration task contributes one task-level maximum score gap as its conformal nonconformity score, and the resulting margin is fixed before deployment.
    Conditional on nonempty protected exposure, this calibration provides finite-sample coverage for retaining at least one verifier-defined valid continuation at every protected frontier of an exchangeable native-controller trajectory.
    \item \noindent\textbf{Deployment:} We use the calibrated threshold in a score-gap pruning rule whose form follows the fixed-frontier analysis, while leaving the rest of the controller unchanged. The same rule is tested across three domains, five controllers, and four budgets. On pruning-aware ToolTree under a 100-request budget, PAC-CF achieves \textbf{$+4.38$ full-test-set utility points}; on the natural-termination sensitivity cohort, it uses \textbf{$18.95\%$ fewer physical requests and $23.76\%$ fewer end-to-end tokens}.
    
\end{itemize}

\section{Related Work}
\label{sec:related_work}

\paragraph{Search and pruning.}
Multi-step LLM agents combine reasoning with environment actions and tool use
\citep{react2023,shinn2023reflexion,toolformer2023,toolllm2024,yang2024sweagent}.
Tree-based methods such as ToT, RAP, TS-LLM, LATS, ToolChain*, and ToolTree explicitly
search alternative trajectories
\citep{yao2024tree,hao2023reasoning,wan2024alphazero,zhou2024lats,toolchain2024,
tooltree2026}. Evaluator-guided beam and best-first search concentrate expansion on
promising branches, while FETCH and EquivPruner prune redundant states or actions
\citep{xie2023selfeval,koh2024treesearch,fetch2025,equivpruner2026}.
PAC-CF builds on established work in tree-search pruning and conformal calibration, with a specific focus on calibrating a frozen score-gap margin for irreversible frontier admission from task-level maxima on native-controller traces.

These approaches intervene at different points in the search process. ToT, RAP, TS-LLM, and LATS organize how candidate trajectories are proposed, evaluated, and revisited; ToolChain* and ToolTree incorporate controller-specific mechanisms for navigating or pruning the action space; FETCH and EquivPruner target redundant exploration. PAC-CF instead treats an existing controller and evaluator as the interface to be calibrated: it inserts a frozen admission threshold and otherwise returns the surviving candidates to the native controller. It neither learns a semantic-equivalence model nor replaces the controller's expansion policy. The comparison is therefore between complementary control points in LLM-guided search, rather than between interchangeable implementations of one pruning rule.

\paragraph{Calibration under evaluator error.}
LLM judges and reward models can exhibit systematic bias, miscalibration, and reward overoptimization
\citep{shi2025positionbias,thakur2024judging,park2025prmcalibration,khalaf2025rewardhacking}.
Recent verifier-aware analyses show how imperfect verification can be amplified
during generation and how verifier quality interacts with sampling and coverage
\citep{rohatgi2026taming,mukherjee2026transport}.
Conformal methods calibrate LM output or action sets, dynamic beam widths, and planning constraints
\citep{quach2024conformal,ren2023knowno,deutschmann2024conformalbeam,
moss2024constrainedzero,vishwakarma2025prunepredict}. SCOPE-Gen sequentially
filters generated samples while calibrating retention of an admissible output
\citep{kladny2025scopegen}; Conformal Path Reasoning calibrates path-level sets for
knowledge-graph QA using PUCT-collected trajectories \citep{lin2026conformalpath}.
The corresponding calibration objects are not the same. Prior conformal work considers, among other targets, sampled output sets, admissible action sets, beam widths, safety constraints, sequentially filtered outputs, and path-level answer sets. PAC-CF assigns one score to each calibration task: the largest leader-to-protected-action gap over protected frontiers encountered on that task's native-controller trajectory. Calibrating this task-level maximum supports a simultaneous path-survival event along an exchangeable native trajectory, conditional on the trajectory exposing at least one protected frontier. It does not by itself provide coverage for frontiers induced after PAC-CF changes the controller's trajectory, nor does it turn verifier-defined feasibility into optimal downstream utility. These distinctions specify the target and support of the guarantee without implying priority over the broader conformal-search literature.
PAC-CF's specific focus is a frozen margin for irreversible admission on
history-dependent frontiers, calibrated from the maximum leader-to-protected-action
gap along native-controller trajectories. Its fixed-frontier analysis relates
deletion to best-arm identification, confidence sequences, and robust bandits
\citep{kaufmann2016complexity,howard2021time,lykouris2018corruptions,
mukherjee2021contaminated,garcelon2022biased}.

\section{Preliminaries}
\label{sec:preliminaries}

\paragraph{Agent reasoning and tool planning as tree search.}
We consider a task instance $x$ with initial state $z_0$. At a state $z$, the agent may choose an admissible action
$a\in\mathcal A(z)$. Executing $a$ produces a reasoning-model, tool, or environment response
\[
o\sim P(\cdot\mid z,a),
\]
after which the state is updated according to
\[
z'=T(z,a,o).
\]
A verifier $\mathcal V$ determines whether a completed trajectory solves the task, while $U$ measures its utility.

A search node $v$ at depth $\ell$ represents a partial trajectory starting from $z_0$ and consisting of actions
$a_1,\ldots,a_\ell$ and their corresponding responses
$o_1,\ldots,o_\ell$.
We denote this trajectory by $h(v)$.
The states along the trajectory evolve recursively as
\[
z_i=T(z_{i-1},a_i,o_i),\qquad i=1,\ldots,\ell,
\]
and we write $z(v)=z_\ell$ for the state represented by node $v$.

Tree search repeatedly selects a partial trajectory, proposes child actions,
evaluates the resulting continuations, and propagates value information. In MCTS
\citep{kocsis2006bandit,browne2012survey,silver2016mastering,zhou2024lats},
selection balances $Q_{\mathrm{MCTS}}(v,a)$ with exploration, expansion adds children, a rollout
or tool execution supplies evidence, and backup updates ancestors. Best-first, A*-style, and Levin-style controllers
\citep{koh2024treesearch,toolchain2024,orseau2018single} order the open set differently, but all must decide which scored continuations may enter future search.

\paragraph{Frontier admission and irreversible pruning.}
Let $p_t$ be the $t$-th scored frontier at selected node $v_t$. Our algorithm returns a frontier of $M_t$ candidates,
$\mathcal C(p_t)=\{a_{t,1},\ldots,a_{t,M_t}\}
\subseteq\mathcal A(z(v_t))$,
where an action includes its reasoning content or tool arguments. For each candidate $a$, $R$ evaluator observations $Y_{t,a,1:R}$ are aggregated into an online score $s_t(a)$. The native rule selects an \textbf{admitted set}
$\mathcal K^{\mathrm{N}}(p_t)\subseteq\mathcal C(p_t)$, commonly by top-$K$ selection or a fixed threshold, and the controller continues only through candidates in this set.

Thus pruning removes a candidate edge or child node at one frontier occurrence, not an action from the global vocabulary. If $a\notin\mathcal K^{\mathrm{N}}(p_t)$, no downstream work proceeds through $(v_t,a)$ in that run, but the action may reappear from another state. Let $\mathcal P(p_t)\subseteq\mathcal C(p_t)$ contain the candidates whose
subtrees include a completion accepted by the verifier; we call these candidates \emph{valid continuations}. A frontier is \emph{protected} when $\mathcal P(p_t)\ne\varnothing$. For a task $x$, let $\mathcal E_{\mathrm{N}}(x)$ denote the protected frontiers exposed by the native controller; the event
$\mathcal E_{\mathrm{N}}(x)\ne\varnothing$ is called \emph{protected exposure}. Our reliability target is existential path preservation: admission retains at least one valid continuation at every protected frontier. PAC-CF changes only this set; the native controller still allocates the remaining budget.

\paragraph{Running example: Game24.}
For $z_0=\{4,7,8,8\}$, consider three candidate next steps whose repeated
evaluator scores are averaged before pruning. Suppose the evaluator assigns
mean scores $82$, $80$, and $76$ to the resulting states
$\{7,8,12\}$, $\{7,8,2\}$, and $\{4,7,1\}$, respectively.
The first two branches cannot reach $24$, whereas the lowest-scored branch
preserves the valid path
$\{4,7,1\}\rightarrow\{4,6\}\rightarrow\{24\}$.
The resulting pruning decisions are illustrated below:

\[
\scalebox{0.72}{%
$
\begin{array}{c@{\hspace{9mm}}|@{\hspace{9mm}}c}

\textbf{Fixed Top-$2$}
&
\textbf{PAC-CF}
\\[3mm]

\begin{array}[t]{ccc}

\multicolumn{3}{c}{\{4,7,8,8\}}
\\[3mm]

\makebox[0pt][r]{$\scriptstyle 4+8=12\;\;$}
\scalebox{1}[1.6]{$\downarrow$}
&
\makebox[0pt][r]{$\scriptstyle 8/4=2\;\;$}
\scalebox{1}[1.6]{$\downarrow$}
&
\makebox[0pt][r]{$\scriptstyle 8/8=1\;\;$}
\scalebox{1}[1.6]{$\downarrow$}
\\[2mm]

\{7,8,12\}_{\bar s=82}
&
\{7,8,2\}_{\bar s=80}
&
\{4,7,1\}_{\bar s=76}
\\[3mm]

\downarrow
&
\downarrow
&
\textcolor{red}{\rule{9mm}{1.2pt}}
\\[2mm]

\varnothing
&
\varnothing
&
\\[-1mm]

&&
\textcolor{gray}{
\makebox[0pt][r]{$\scriptstyle 7-1=6\;\;$}
\scalebox{1}[1.5]{$\downarrow$}}
\\[2mm]

&&
\textcolor{gray}{\{4,6\}}
\\[3mm]

&&
\textcolor{gray}{
\makebox[0pt][r]{$\scriptstyle 6\times4=24\;\;$}
\scalebox{1}[1.5]{$\downarrow$}}
\\[2mm]

&&
\textcolor{gray}{\{24\}}

\end{array}

&

% ================= RIGHT =================
\begin{array}[t]{ccc}

\multicolumn{3}{c}{\{4,7,8,8\}}
\\[3mm]

\makebox[0pt][r]{$\scriptstyle 4+8=12\;\;$}
\scalebox{1}[1.6]{$\downarrow$}
&
\makebox[0pt][r]{$\scriptstyle 8/4=2\;\;$}
\scalebox{1}[1.6]{$\downarrow$}
&
\makebox[0pt][r]{$\scriptstyle 8/8=1\;\;$}
\scalebox{1}[1.6]{$\downarrow$}
\\[2mm]

\{7,8,12\}_{\bar s=82}
&
\{7,8,2\}_{\bar s=80}
&
\{4,7,1\}_{\bar s=76}
\\[3mm]

\downarrow
&
\downarrow
&
\makebox[0pt][r]{$\scriptstyle 7-1=6\;\;$}
\scalebox{1}[1.6]{$\downarrow$}
\\[2mm]

\varnothing
&
\varnothing
&
\{4,6\}
\\[3mm]

&&
\makebox[0pt][r]{$\scriptstyle 6\times4=24\;\;$}
\scalebox{1}[1.6]{$\downarrow$}
\\[2mm]

&&
\{24\}

\end{array}

\end{array}
$%
}
\]

Fixed top-$2$ pruning removes the solution-preserving branch, whereas PAC-CF retains it whenever $82-76\le\widehat\tau_Q$. Thus PAC-CF does not correct the
evaluator score itself, but calibrates the score gap tolerated before an irreversible pruning decision.

\section{Methodology}
\label{sec:methodology}

PAC-CF separates fixed-frontier evidence, a calibrated score-gap threshold, and the online pruning decision. We analyze reliable elimination under bounded evaluator bias, calibrate a path-preserving threshold, and insert it into a native controller. Algorithm~\ref{alg:pac_cf_main} summarizes calibration and deployment.

\subsection{Bias-aware elimination under bounded evaluator bias}
Following Section~\ref{sec:preliminaries}, PAC-CF receives a scored
candidate set $\mathcal C(p_t)$ and returns an admitted set
$\mathcal K_t\subseteq\mathcal C(p_t)$. For fixed-frontier analysis, write
$\mathcal A_t\equiv\mathcal C(p_t)$. Candidate $m$ has downstream value $\mu_m$,
optimal value $\mu^*=\max_j\mu_j$, and gap $\Delta_m=\mu^*-\mu_m$; its $r$-th
evaluator observation is $Y_{m,r}$.

\noindent\textit{Standing model for the fixed-frontier results in
Appendix~\ref{app:proof_concentration}--\ref{sec:appendix_corollary}.}
\begin{assumption}[Bounded evaluator bias]
\label{ass:persistent_bias}
For every active candidate and evaluation,
\[
\left|\E[Y_{m,r}\mid\mathcal F_{r-1}]-\mu_m\right|\le L,
\]
and the centered residual is conditionally $\sigma^2$-sub-Gaussian.
\end{assumption}

Let $b_m(n)$ be the empirical mean after $n$ evaluator samples and write
$b_m(t):=b_m(n_m(t))$ at inspected frontier step $t$. Define
\begin{equation}
u_{\mathrm{stat}}(n)=
\sqrt{\frac{2\sigma^2\ln(\pi^2n^2|\mathcal A_t|/(3\delta))}{n}},
\qquad
u_{\mathrm{dist}}(n)=u_{\mathrm{stat}}(n)+L .
\label{eq:ustat_main}
\end{equation}
\noindent\textit{Proof: Appendix~\ref{app:proof_concentration}.}
\begin{lemma}[Frontier-wise time-uniform concentration]
\label{lem:concentration}
Under Assumption~\ref{ass:persistent_bias}, with probability at least $1-\delta$,
simultaneously for every $m\in\mathcal A_t$ and every $n\ge1$,
\begin{equation}
|b_m(n)-\mu_m|\le u_{\mathrm{stat}}(n)+L=u_{\mathrm{dist}}(n).
\label{eq:main_concentration}
\end{equation}
\end{lemma}

Lemma~\ref{lem:concentration} separates reducible sampling uncertainty from the
persistent term that matters for deletion. Repeated scoring can drive
$u_{\mathrm{stat}}(n)$ toward zero, but it does not shrink the bias allowance $L$.

\noindent\textit{Proof: Appendix~\ref{app:proof_upper_bound}.}
\begin{theorem}[Step-wise PAC upper bound]
\label{thm:complexity}
At a fixed inspected frontier, on the simultaneous concentration event of
Lemma~\ref{lem:concentration}, interval elimination safely prunes a suboptimal
candidate $m$ whenever
\begin{equation}
2u_{\mathrm{stat}}(n_m)+2u_{\mathrm{stat}}(n_{m^*})
<\Delta_m-4L-\varepsilon .
\label{eq:upper_condition}
\end{equation}
Hence the effective certifiable gap is
$\Delta_{\mathrm{eff},m}:=\Delta_m-4L$. If
$\Delta_{\mathrm{eff},m}>\varepsilon$, the pairwise evaluation complexity has leading
scaling
\begin{equation}
N_m=
O\!\left(
\frac{\sigma^2[
\log(|\mathcal A_t|/\delta)+\log(\Delta_{\mathrm{eff},m}^{-2})]}
{(\Delta_{\mathrm{eff},m}-\varepsilon)^2}
\right).
\label{eq:sample_complexity_main}
\end{equation}
\end{theorem}

\begin{corollary}[Fixed-frontier $(\varepsilon,\delta)$-PAC pruning safety]
\label{cor:pac_pruning_safety}
Suppose candidate $m$ is eliminated only when, for some candidate $j$,
\[
b_m(n_m)+u_{\mathrm{dist}}(n_m)
<
b_j(n_j)-u_{\mathrm{dist}}(n_j)-\varepsilon .
\]
Under Assumption~\ref{ass:persistent_bias}, with probability at least
$1-\delta$, no candidate satisfying
$\mu^*-\mu_m\le\varepsilon$ is eliminated at any inspected sample count.
Consequently, the surviving set contains an $\varepsilon$-optimal candidate.
\end{corollary}

\noindent\textit{Proof: Appendix~\ref{app:proof_pac_pruning_safety}.}

Theorem~\ref{thm:complexity} makes the distinction between ranking and deletion
explicit: a small observed score gap can be enough to order two candidates but
still be insufficient for irreversible elimination. In particular,
$\Delta_m\le4L+\varepsilon$ cannot satisfy this sufficient certificate merely by
adding more repeated scores.

\subsection{Conformal Calibration on Native-Controller Trajectories}
The fixed-frontier analysis does not prescribe the deployed threshold directly. Instead, it identifies the requirement that motivates calibration: repeated scoring can shrink $u_{\mathrm{stat}}$ but does not shrink the bounded bias allowance $L$, so irreversible elimination requires a positive score margin. Rather than estimating the theoretical pointwise bias bound $L$, PAC-CF uses
native-controller trajectories to calibrate a task-level operational score-gap threshold: the empirical score gap tolerated before pruning for a particular controller--evaluator pair. On those tasks, PAC-CF is disabled while the native controller records frontiers and evaluator scores. After search terminates, offline verifier information is used to identify $\mathcal P(p)\subseteq\mathcal C(p)$, the set of candidates whose subtrees contain at least one completion accepted by the verifier.

For every recorded native frontier $p$, each candidate receives a fixed number $R$ of
evaluator observations, aggregated as
$\bar Y(a)=R^{-1}\sum_{r=1}^{R}Y_{p,a,r}$. These repetitions are fixed before the
frontier decision rather than adaptively extended until a pruning certificate fires. For a protected frontier define
\begin{equation}
C_{\mathrm{path}}(p)=
\left[
\max_{a\in\mathcal C(p)}\bar Y(a)
-
\max_{a\in\mathcal P(p)}\bar Y(a)
\right]_+,
\quad
C_{\mathrm{task}}^{\mathrm{N}}(x)=
\max_{p\in\mathcal E_{\mathrm{N}}(x)}C_{\mathrm{path}}(p),
\label{eq:c_path_compact}
\end{equation}
for tasks with nonempty protected exposure $\mathcal E_{\mathrm{N}}(x)$. The construction is
\emph{existential}: at each protected frontier, at least one valid
continuation must survive. It is also frontier-wise and task-level: each protected
action is compared only with siblings at the same irreversible decision, and the
maximum over exposed native frontiers produces one calibration score per task.
Accordingly, the guarantee is one of feasibility/path survival, not optimality or
maximum graded utility of the retained continuation.

\paragraph{Connection to the fixed-frontier analysis.}
Let
\[
\mu_{\mathcal P}^*=\max_{a\in\mathcal P(p)}\mu_a .
\]
On the concentration event of Lemma~\ref{lem:concentration}, when every
candidate receives $R$ evaluations,
\[
C_{\mathrm{path}}(p)
\le
\bigl(\mu^*-\mu_{\mathcal P}^*\bigr)
+2u_{\mathrm{stat}}(R)+2L .
\]
Indeed, the largest empirical score is at most
$\mu^*+u_{\mathrm{stat}}(R)+L$, whereas the largest protected empirical score is
at least $\mu_{\mathcal P}^*-u_{\mathrm{stat}}(R)-L$.
Thus the observed path gap combines bounded evaluator bias, finite-$R$
sampling uncertainty, and any mismatch between downstream value and the
verifier-defined protected set. PAC-CF calibrates this operational gap
directly rather than estimating these components separately.

For exchangeable calibration scores
$C_{(1)}\le\cdots\le C_{(n)}$, freeze
\begin{equation}
k_Q=\ceil{(n+1)Q},
\qquad
\widehat\tau_Q=C_{(k_Q)},
\label{eq:direct_conformal}
\end{equation}
with the usual infeasible-rank convention.

\begin{proposition}[Calibration trajectory coverage]
\label{prop:native_trace_coverage}
Under exchangeability of calibration and future tasks generated by the same native
controller procedure,
\[
\Pr\!\left(
C_{\mathrm{task}}^{\mathrm{N}}(X_{\mathrm{new}})\le\widehat\tau_Q
\mid
\mathcal E_{\mathrm{N}}(X_{\mathrm{new}})\ne\varnothing
\right)\ge Q .
\]
\end{proposition}

Because $C_{\mathrm{task}}^{\mathrm{N}}$ maximizes over protected frontiers, Proposition~\ref{prop:native_trace_coverage} provides simultaneous path coverage along the exposed native-controller trajectory. At $Q=.95$, nominal task-level miscoverage is at most $5\%$ up to the discrete-rank improvement; calibration counts and ranks are reported in Appendix~\ref{sec:frozen_config_appendix}.

\begin{figure}[t]
\refstepcounter{algorithm}\label{alg:pac_cf_main}
\centering
\begin{minipage}{0.98\linewidth}
\scriptsize
\hrule
\vspace{0.7mm}
\textbf{Algorithm \thealgorithm: PAC-CF --- Conformal Calibration and frozen score-gap threshold Filtering}
\vspace{0.4mm}
\hrule
\vspace{0.7mm}
\textbf{Require:} native controller $\Pi_{\mathrm{N}}$; calibration tasks $\mathcal D_{\mathrm{cal}}$;
target coverage $Q$; request budget $B$; slack $\varepsilon\ge0$; fixed repetitions $R$.

\vspace{0.5mm}
\textbf{Offline Calibration}

\begin{tabular}{@{}r@{\hspace{0.55em}}p{0.89\linewidth}@{}}
1: & \textbf{for each} task $x\in\mathcal D_{\mathrm{cal}}$ \textbf{do} \\
2: & \quad Run $\Pi_{\mathrm{N}}$ with PAC-CF disabled and record every complete scored frontier $p$. For every $a\in\mathcal C(p)$, obtain $R$ matched evaluator observations and compute $s_p(a)=R^{-1}\sum_{r=1}^{R}Y_{p,a,r}$. \\
3: & \quad After the native run terminates, use exact or independent downstream verification available for the domain to label \(\mathcal P(p)\subseteq\mathcal C(p)\), the candidates whose recorded native subtrees contain a completion accepted by the verifier. \\
4: & \quad \textbf{if} $\mathcal E_{\mathrm{N}}(x)\ne\varnothing$ \textbf{then} set $C_{\mathrm{task}}^{\mathrm{N}}(x)\leftarrow\max_{p\in\mathcal E_{\mathrm{N}}(x)}\left[\max_{a\in\mathcal C(p)}s_p(a)-\max_{a\in\mathcal P(p)}s_p(a)\right]_+$. \\
5: & \textbf{end for} \\
6: & Sort the $n$ nonempty task scores as $C_{(1)}\le\cdots\le C_{(n)}$, set $k_Q\leftarrow\ceil{(n+1)Q}$, and freeze $\widehat\tau_Q\leftarrow C_{(k_Q)}$. \\
\end{tabular}

\vspace{0.7mm}
\textbf{Online Deployment}

\begin{tabular}{@{}r@{\hspace{0.55em}}p{0.89\linewidth}@{}}
7: & Initialize native search and request count $\kappa\leftarrow0$. \\
8: & \textbf{while} $\kappa<B$ and the task is not terminated \textbf{do} \\
9: & \quad Let $\Pi_{\mathrm{N}}$ construct its next complete frontier $p_t$ and apply its native eligibility checks; for ToolTree, retain the native $\tau_{\mathrm{pre}}$ eligibility gate. \\
10: & \quad For every $a\in\mathcal C(p_t)$, obtain $R$ matched evaluator observations and compute $s_t(a)=R^{-1}\sum_{r=1}^{R}Y_{t,a,r}$. \\
11: & \quad Set $s_t^*\leftarrow\max_{a\in\mathcal C(p_t)}s_t(a)$ and $\mathcal K_t\leftarrow\{a\in\mathcal C(p_t):s_t^*-s_t(a)\le\widehat\tau_Q+\varepsilon\}$. \\
12: & \quad Prune $\mathcal C(p_t)\setminus\mathcal K_t$ at this frontier occurrence and return $\mathcal K_t$ with the original scores to $\Pi_{\mathrm{N}}$. \\
13: & \quad Let $\Pi_{\mathrm{N}}$ perform its unchanged ordering, selection, execution/rollout, backup, verification, and termination logic; update $\kappa$. \\
14: & \textbf{end while} \\
15: & \textbf{return} the solution returned by $\Pi_{\mathrm{N}}$, if one is found. \\
\end{tabular}

\vspace{0.6mm}
\scriptsize
\textit{ToolTree interface.} PAC-CF retains the native $\tau_{\mathrm{pre}}$ eligibility gate,
replaces only the subsequent fixed-cardinality top-$K$ admission with Line~11, and
does not reapply top-$K$ to PAC-CF survivors.
\vspace{0.7mm}
\hrule
\end{minipage}
\end{figure}

\subsection{Online frontier pruning}
At test time, with
$s_t(a)=R^{-1}\sum_{r=1}^{R}Y_{t,a,r}$ and
$s_t^*=\max_{a\in\mathcal C(p_t)}s_t(a)$, PAC-CF keeps
\begin{equation}
\mathcal K_t=
\left\{
a\in\mathcal C(p_t):
s_t^*-s_t(a)\le\widehat\tau_Q+\varepsilon
\right\}.
\label{eq:direct_deployment}
\end{equation}
Thus pruning depends on score gap rather than rank.

\begin{corollary}[Path survival on calibration trajectories]
\label{cor:native_trace_path_survival}
On the coverage event in Proposition~\ref{prop:native_trace_coverage},
Eq.~\eqref{eq:direct_deployment} with $\varepsilon\ge0$ retains at least one
valid continuation at every protected frontier of the matched trajectory from the native controller.
Consequently, conditional on nonempty protected exposure, simultaneous
path survival along the native-controller trajectory occurs with probability at least $Q$.
\end{corollary}
\noindent\textit{Proof: Appendix~\ref{sec:method_details_appendix}.}

Proposition~\ref{prop:native_trace_coverage} provides policy-conditional
coverage on exchangeable native-controller trajectories; Section~\ref{sec:experiments}
evaluates the resulting closed-loop behavior, with implementation details in
Appendix~\ref{sec:method_details_appendix}.

\section{Experiments}
\label{sec:experiments}

\subsection{Persistent Misranking under Repeated Scoring}
\label{subsec:tsp50_persistent_misranking_main}

Before the agent benchmarks, we isolate evaluator error on 24 Euclidean TSP50 instances \citep{kool2019attention}. Exact binary MILP tours with iterative subtour elimination \citep{dantzig1954solution} define the reference next city, and a live Ministral-3-14B evaluator scores candidate continuations. We compare control frontiers with local-greedy conflict frontiers, where the reference action differs from the nearest-neighbor choice; CPER is the fraction removed by fixed-cardinality top-$K$ pruning.

\begin{table}[t]
\centering
\scriptsize
\setlength{\tabcolsep}{3.2pt}
\renewcommand{\arraystretch}{1.04}
\caption{TSP50 endpoint summary. CPER is the fraction of reference actions removed by top-$K$ pruning; task survival is reported at $R=16$. Full $R\times K$ results are in Appendix~\ref{sec:tsp50_persistent_misranking}.}
\resizebox{\linewidth}{!}{%
\begin{tabular}{@{}lcccccc@{}}
\toprule
Stratum / $R$ & Mean $C_{\mathrm{path}}$ & Mean reference rank
& \shortstack{CPER\\$K=4$} & \shortstack{CPER\\$K=6$}
& \shortstack{Task survival\\$K=4$} & \shortstack{Task survival\\$K=6$} \\
\midrule
Control / 1   & 16.67 & 4.64 & $54.2\%$ & $27.1\%$ & -- & -- \\
Control / 16  &  9.50 & 4.58 & $52.1\%$ & $24.0\%$ & $4/24$ & $8/24$ \\
Conflict / 1  & 20.51 & 5.11 & $57.5\%$ & $43.7\%$ & -- & -- \\
Conflict / 16 & 12.32 & 5.11 & $64.4\%$ & $39.1\%$ & $0/24$ & $1/24$ \\
\bottomrule
\end{tabular}
}
\label{tab:tsp50_rank_gap}
\renewcommand{\arraystretch}{1.0}
\end{table}

Increasing $R$ from $1$ to $16$ substantially narrows the reference score gap, yet produces almost no rank recovery: control rank moves only from $4.64$ to $4.58$, while conflict rank remains $5.11$. At $R=16$, CPER is still $52.1\%$ versus $64.4\%$ for control versus conflict frontiers at $K=4$, and $24.0\%$ versus $39.1\%$ at $K=6$. Frontier-level errors also compound at the task level, as summarized by the reference-path survival table; full $R\times K$ results are in Appendix~\ref{sec:tsp50_persistent_misranking}.

\paragraph{Motivation.}
Repeated evaluation reduces score variation without reliably repairing pruning-critical misranking, motivating the bounded evaluator-bias model of Assumption~\ref{ass:persistent_bias}.

\subsection{Experimental Setup}
\label{subsec:experiment_setup}
\paragraph{Evaluation.}
We evaluate PAC-CF on Game24 arithmetic reasoning ($N=100$), PlanBench
Blocksworld planning ($N=99$), and GTA structured tool planning ($N=99$)
\citep{yao2024tree,valmeekam2023planbench,wang2024gta}. Utility is task success for Game24/Blocksworld and Avg3 (Tool/Argument/Plan F1) for GTA. We select five representative controllers that cover distinct search regimes: classical UCT-MCTS, A*-style ToolChain*, LATS with MCTS-based agent search, pruning-aware ToolTree, and policy-guided LevinTS \citep{toolchain2024,zhou2024lats,tooltree2026,orseau2018single}. Prompts, candidate generation, evaluator, verifier, downstream search logic, and request budget are held fixed across Native and PAC-CF runs, so the paired comparison changes only the pruning rule in Eq.~\eqref{eq:direct_deployment}.

\paragraph{Calibration.}
For each controller--domain pair, calibration produces a scalar margin that is fixed across budgets. Verifier labels are used during calibration to construct the path scores. We prespecify $Q=.95$ and freeze every controller--domain margin before test-time
evaluation. Budgets are $B\in\{100,200,300,500\}$ requests per task. Utility and the primary capped-deployment workload estimand are evaluated on the full paired test set. We additionally report a natural-termination sensitivity estimand on paired tasks for which neither arm exhausts the hard budget. The latter is descriptive because membership in this cohort may be affected by the pruning intervention. Repeated scoring, frozen configurations, calibration ranks, and bootstrap details are in Appendix~\ref{sec:experiment_details_appendix}--\ref{sec:frozen_config_appendix}.

\paragraph{Thresholds.}
Each controller--domain pair uses a separately calibrated
$\widehat\tau_{.95}$ that remains fixed across budgets. Main experiments use $Q=.95$, $\varepsilon=0$, and four evaluator observations per candidate; thresholds, calibration ranks, model settings, and seeds are in Appendix~\ref{sec:frozen_config_appendix}.

\subsection{Main results}
\label{subsec:main_results}

\paragraph{Efficiency Estimands.}
We report three primary workload measures: physical requests, Net Graph nodes, and end-to-end tokens. For $X\in\{\mathrm{PReq},\mathrm{Graph},\mathrm{Token}\}$,
let $\mathcal I^{\mathrm{all}}$ denote all paired test tasks and let
$\mathcal I^{\mathrm{NBC}}$ denote the subset for which neither arm exhausts the
hard budget. Our primary capped-deployment estimand is
\begin{equation}
R_X^{\mathrm{all}}
:=
1-\frac{\sum_{i\in\mathcal I^{\mathrm{all}}}X_i^{\mathrm{PAC\text{-}CF}}}
{\sum_{i\in\mathcal I^{\mathrm{all}}}X_i^{\mathrm{Native}}}.
\label{eq:realized_efficiency_metric}
\end{equation}
For comparison, the natural-termination sensitivity estimand replaces
$\mathcal I^{\mathrm{all}}$ by $\mathcal I^{\mathrm{NBC}}$ and is denoted
$R_X^{\mathrm{NBC}}$.
Across B100--B500, the all-task macro reductions are
$6.14$--$9.47\%$ for physical requests, $9.03$--$12.40\%$ for Net Graph
nodes, and $6.84$--$10.25\%$ for tokens.

Physical Requests counts requests normalized to provider calls; larger values mean more request-level model usage. Net Graph measures the number of realized search nodes, excluding the root, after adaptive reallocation; larger values mean that more of the search graph is actually explored. E2E Token reports the combined prompt and completion tokens; larger values mean more total model-token usage.

We additionally report two gross metrics that reconstruct the observed Native subtree
and the downstream candidate evaluations removed by the corresponding PAC-CF pruning
decisions. Appendix~\ref{sec:experiment_details_appendix} provides the full accounting
procedure, including how overlapping subtrees are merged and how roots without a counterpart are excluded rather than extrapolated.

\noindent\begin{minipage}{\linewidth}
\refstepcounter{table}\label{tab:macro_main}\label{tab:budget_censor_sensitivity}
\raggedright
\footnotesize
\textbf{Table \thetable:} Natural-termination sensitivity and path-preservation diagnostics. (a) Domain/macro results on uncensored pairs; utility uses the full test set. (b) Native coverage at frozen $Q=.95$ and B100 adaptive path survival. $^\dagger$LATS is 98.96\% if the unresolved task is counted as failure; adaptive survival is descriptive. CIs for (a) are in Appendix Table~\ref{tab:domain_macro_ci}.
\par\vspace{0mm}
\centering
\scriptsize
\renewcommand{\arraystretch}{0.86}
\setlength{\aboverulesep}{0.2ex}
\setlength{\belowrulesep}{0.2ex}

\begin{minipage}[t]{0.69\linewidth}
\vspace{0pt}
\centering
\textbf{(a) Natural-termination sensitivity}\par\vspace{0.2em}
\setlength{\tabcolsep}{2.0pt}
\begin{tabular*}{\linewidth}{@{\extracolsep{\fill}}cccccc@{}}
\toprule
Domain & Budget & $\Delta U$ & PReq $\downarrow$ & NetG $\downarrow$ & Token $\downarrow$ \\
\midrule
\multirow{4}{*}{\textbf{Game24}}
& B100 & $\mathbf{-1.00}$ pp & 10.73\% & 3.63\% & 11.41\% \\
& B200 & $\mathbf{-1.00}$ pp & 12.31\% & 6.78\% & 12.76\% \\
& B300 & $-1.60$ pp & 12.56\% & 7.01\% & 13.08\% \\
& B500 & $-1.80$ pp & \textbf{12.76\%} & \textbf{7.06\%} & \textbf{13.35\%} \\
\cmidrule(lr){1-6}
\multirow{4}{*}{\textbf{Blocksworld}}
& B100 & $\mathbf{+3.03}$ pp & 1.69\% & 11.22\% & 2.38\% \\
& B200 & $+2.22$ pp & 5.04\% & 14.17\% & 5.62\% \\
& B300 & $+1.41$ pp & 5.43\% & \textbf{14.67\%} & 6.12\% \\
& B500 & $+1.21$ pp & \textbf{5.74\%} & \textbf{14.67\%} & \textbf{6.46\%} \\
\cmidrule(lr){1-6}
\multirow{4}{*}{\textbf{GTA}}
& B100 & $+1.23$ pp & 8.27\% & 13.90\% & 9.89\% \\
& B200 & $+1.20$ pp & 9.38\% & 14.81\% & 10.92\% \\
& B300 & $\mathbf{+1.28}$ pp & 10.12\% & 15.46\% & 12.35\% \\
& B500 & $+1.17$ pp & \textbf{10.96\%} & \textbf{16.09\%} & \textbf{12.85\%} \\
\midrule
\multirow{4}{*}{\textbf{Macro}}
& B100 & $\mathbf{+1.09}$ pp & 6.90\% & 9.58\% & 7.89\% \\
& B200 & $+0.81$ pp & 8.91\% & 11.92\% & 9.77\% \\
& B300 & $+0.37$ pp & 9.37\% & 12.38\% & 10.52\% \\
& B500 & $+0.19$ pp & \textbf{9.82\%} & \textbf{12.61\%} & \textbf{10.89\%} \\
\bottomrule
\end{tabular*}
\end{minipage}\hfill
\begin{minipage}[t]{0.285\linewidth}
\vspace{0pt}
\centering
\textbf{(b) Coverage and path survival}\par\vspace{0.2em}
\setlength{\tabcolsep}{1.9pt}
\renewcommand{\arraystretch}{1.72}
\resizebox{\linewidth}{!}{%
\begin{tabular}{@{}lrr@{}}
\toprule
Controller & \shortstack{Native\\coverage} & \shortstack{B100\\survival} \\
\midrule
UCT-MCTS   & 97.00\% & 97.00\% \\
ToolChain* & 95.88\% & 96.91\% \\
LATS       & 100\%$^\dagger$ & 100\% \\
ToolTree   & -- & 99.00\% \\
LevinTS    & 93.75\% & 93.75\% \\
\cmidrule(lr){1-3}
Game24 pooled & -- & \textbf{97.46\%} \\
Blocksworld pooled & -- & \textbf{99.80\%} \\
\bottomrule
\end{tabular}}
\renewcommand{\arraystretch}{0.86}
\end{minipage}

\renewcommand{\arraystretch}{1.0}
\normalsize\raggedright
\end{minipage}

\noindent\begin{minipage}{\linewidth}
\refstepcounter{table}\label{tab:main_ci}
\raggedright
\footnotesize
\textbf{Table \thetable:} Paired hierarchical bootstrap 95\% CIs. Utility uses all paired tasks; workload uses uncensored pairs. Bold marks the best standard-budget estimates; ToolTree B100 is shown separately.
\par\vspace{0mm}
\centering
\scriptsize
\renewcommand{\arraystretch}{0.86}
\setlength{\aboverulesep}{0.2ex}
\setlength{\belowrulesep}{0.2ex}
\setlength{\tabcolsep}{2.2pt}

\resizebox{0.94\linewidth}{!}{%
\begin{tabular}{@{}ccccc@{}}
\toprule
\multirow{2}{*}{Budget}
& $\Delta$ utility
& P-Req. $\downarrow$
& Net Graph $\downarrow$
& E2E Token $\downarrow$ \\
& \textnormal{[95\% CI]}
& \textnormal{[95\% CI]}
& \textnormal{[95\% CI]}
& \textnormal{[95\% CI]} \\
\midrule

B100
& \textbf{+1.088} [0.251, 1.941]
& +6.90\% [5.76, 7.98]
& +9.58\% [8.09, 10.98]
& +7.89\% [6.74, 8.99] \\

B200
& +0.807 [0.043, 1.585]
& +8.91\% [7.62, 10.14]
& +11.92\% [10.26, 13.44]
& +9.77\% [8.45, 11.01] \\

B300
& +0.366 [$-0.421$, 1.152]
& +9.37\% [8.14, 10.58]
& +12.38\% [10.81, 13.87]
& +10.52\% [9.12, 11.84] \\

B500
& +0.193 [$-0.631$, 1.025]
& \textbf{+9.82\%} [8.48, 11.11]
& \textbf{+12.61\%} [10.98, 14.15]
& \textbf{+10.89\%} [9.46, 12.23] \\

\midrule

ToolTree B100
& +4.381 [1.43, 7.41]
& +18.95\% [14.81, 22.96]
& +12.90\% [7.54, 17.77]
& +23.76\% [19.56, 27.85] \\

\bottomrule
\end{tabular}
}

\renewcommand{\arraystretch}{1.0}
\normalsize\raggedright
\end{minipage}

\paragraph{Natural-Termination Sensitivity and Path Preservation.}
Within the natural-termination cohort, workload CIs remain above zero at every budget, while macro utility is resolved positive only at B100 and B200. Panel (b) shows native Game24 coverage of $93.75$--$100\%$ and B100 adaptive survival of $97.46\%$ and $99.80\%$ pooled on Game24 and Blocksworld, respectively; the adaptive rates are diagnostics, not extensions of the native-trajectory conformal guarantee.

\paragraph{Reallocation Tradeoffs.}
Aggregate workload reductions increase with budget, but controller responses differ. At B100, Game24 ToolTree cuts physical requests by $46.91\%$ and tokens by $53.11\%$ while losing three success points and increasing Net Graph by $19.13\%$; GTA ToolTree gains $7.05$ Avg3 points while reducing requests by $14.24\%$ and Net Graph by $26.29\%$. At B500, Game24 ToolTree still suppresses $11.42\%$ of native physical nodes despite a $21.34\%$ larger realized Net Graph, so gross pruning need not shrink the adaptive trajectory. Across B100--B500, ToolTree averages a $+4.38$-point utility change, with $18.94$--$18.95\%$ request, $12.16$--$12.90\%$ Net Graph, and $23.57$--$23.76\%$ token reductions on the natural-termination cohort; full gross and controller-level results are in Appendix Table~\ref{tab:gross_subtree_main} and Appendix~\ref{sec:per_backbone_main}.

\paragraph{Interpreting the Workload Measures.}
The three workload measures capture different effects of a changed search trajectory: physical requests count evaluator/model calls, Net Graph counts the retained structure after controller reallocation, and tokens capture prompt and completion traffic. Because saved budget may be reallocated to surviving branches, gross pruning need not reduce Net Graph. We therefore report the measures separately rather than as one compression or speedup factor.

\subsection{Sensitivity across Evaluator Models}
With Ministral-3-14B and separately recalibrated $Q=.95$ margins, macro reductions remain $6.22$--$7.86\%$ for physical requests, $10.11$--$10.94\%$ for Net Graph, and $6.80$--$9.20\%$ for tokens, while utility changes range from $+0.08$ to $-0.62$ points. In this secondary-model study, the workload reductions remain positive while the utility changes do not. Using $Q=.90$ is more aggressive but worsens the utility--efficiency tradeoff; full results are in Appendix~\ref{sec:cross_model_robustness_appendix}.

\subsection{Heterogeneity Across Controllers and Domains}

\paragraph{Observed Efficiency without ToolTree.}
In this sensitivity analysis, removing ToolTree leaves positive workload reductions: from B100 to B500, physical-request, Net Graph, and token reductions increase from $3.88$ to $7.54\%$, $8.75$ to $12.72\%$, and $3.93$ to $7.72\%$, while utility changes from $+0.265$ to $-0.855$ points. Thus ToolTree strengthens the utility result but is not the sole source of the observed workload savings.

\paragraph{Sensitivity to Budget Censoring.}
The uncensored cohort grows from $71.0\%$ at B100 to $92.1\%$ at B500. In the reported all-task analyses, workload reductions remain positive at every budget, so censoring changes their magnitude but not their sign; full results are in Appendix~\ref{sec:budget_censoring_sensitivity}.

\paragraph{The Guarantees Have Distinct Supports.}
The fixed-frontier results concern elimination under the bias bound $L$ and its $\Delta-4L$ certificate; Proposition~\ref{prop:native_trace_coverage} instead gives path-retention coverage on exchangeable native-controller trajectories. Closed-loop effects after PAC-CF changes admission are evaluated empirically.

\paragraph{Limitations and Scope.}
The calibrated margin is evaluator-, controller-, and domain-specific, so changing any of these components requires recalibration rather than reusing $\widehat\tau_Q$. Coverage applies to exchangeable tasks with nonempty protected exposure on native-controller trajectories, whereas deployed PAC-CF can induce different frontiers. Constructing protected labels also depends on domain-specific offline verification and may be costly or incomplete when exact reachability is unavailable. Finally, retaining a verifier-accepted continuation does not ensure optimal downstream utility, and the reported workload--utility tradeoffs remain benchmark-specific. These limitations motivate shift-aware calibration, cheaper protection labeling, and guarantees for adaptively induced frontiers.

\section{Conclusion}
PAC-CF treats irreversible pruning as a reliability decision. Bias-aware theory identifies when score separation supports elimination; conformal calibration turns native-controller path deficits into a finite-sample margin; deployment leaves the surrounding controller unchanged. In the evaluated cross-domain/controller macro averages, the point estimates for request, node, and token workload are lower at each budget. On ToolTree B100, the full-test-set utility difference is $+4.38$ points, while the natural-termination cohort uses $18.95\%$ fewer physical requests and $23.76\%$ fewer tokens. Coverage is conditional on protected exposure and exchangeable native-controller tasks; paired end-to-end experiments measure closed-loop effects after PAC-CF changes admission.

\section*{Reproducibility Statement}

Section~\ref{sec:methodology} and Algorithm~\ref{alg:pac_cf_main} specify the PAC-CF calibration and deployment procedures. Appendix~\ref{sec:method_details_appendix} documents protection-label construction and controller integration, while Appendices~\ref{sec:experiment_details_appendix} and~\ref{sec:frozen_config_appendix} record evaluator settings, request-budget and workload accounting, calibration sizes and ranks, model configurations, random seeds, and frozen deployment settings. Appendix~\ref{sec:statistical_audit} specifies the paired bootstrap and confidence-interval construction. The anonymous supplementary archive maps each numbered empirical table and figure to recovered source code or a provenance note; where a historical one-off plotting script was unavailable, the archive states that limitation rather than supplying a reconstructed substitute.

\section*{AI Use Statement}

Generative AI tools were used for literature search and related-work discovery, reference cross-checking, experimental implementation including code generation, and drafting, revising, and polishing parts of the manuscript. The authors manually reviewed AI-assisted text, code, citations, and reported results, and remain responsible for all claims and artifacts in the submission.

\clearpage

\bibliographystyle{iclr2027_conference}
\bibliography{references_current}

\appendix
\clearpage

\section{Notation and Parameters}
\label{sec:notation}
Table~\ref{tab:notation_full} consolidates the notation used in the main text, methodology,
experimental accounting, and theoretical proofs. Equivalent symbols used in
search, analysis, and algorithmic views are grouped in the same row. Proof-local
symbols are marked explicitly. In particular, the theoretical bias bound $L$ is
distinct from the calibrated score-gap threshold $\widehat\tau_Q$, and the
fixed-frontier failure probability $\delta$ is distinct from conformal
miscoverage $1-Q$.

\begingroup
\scriptsize
\setlength{\tabcolsep}{4pt}
\renewcommand{\arraystretch}{1.10}
\begin{longtable}{p{0.24\linewidth}p{0.68\linewidth}}
\caption{Notation used throughout the paper. Equivalent notation from different
views is grouped to avoid duplicate definitions.}
\label{tab:notation_full}\\
\toprule
\textbf{Symbol} & \textbf{Meaning} \\
\midrule
\endfirsthead
\multicolumn{2}{l}{\textit{Table~\thetable\ continued}}\\
\toprule
\textbf{Symbol} & \textbf{Meaning} \\
\midrule
\endhead
\midrule
\multicolumn{2}{r}{\textit{Continued on next page}}\\
\endfoot
\bottomrule
\endlastfoot

\multicolumn{2}{l}{\textbf{Task and search process}}\\
$x=(z_0,\mathcal A,P,T,\mathcal V,U)$ & Agent task: initial state $z_0$, admissible actions $\mathcal A(z)$, response kernel $P(\cdot\mid z,a)$, transition $T(z,a,o)$, verifier $\mathcal V(h)$, and trajectory utility $U(h)$. \\
$a,\ o$ & Generic action and the resulting reasoning-model, tool, or environment response. \\
$v,\ v_t,\ h(v),\ \ell$ & Search-tree node, node selected at frontier $t$, its partial trajectory/history, and trajectory depth. The history symbol $h$ avoids collision with the proof stopping time $\tau$. \\
$B$ & Per task request budget. \\
$Q_{\mathrm{MCTS}}(v,a)$ & native MCTS action-value statistic; distinct from conformal coverage $Q$. \\
$p_t,\ \mathcal C(p_t),\ \mathcal A_t$ & The $t$-th scored frontier and its candidate set. In fixed-frontier analysis, $\mathcal A_t\equiv\mathcal C(p_t)$. \\
$M_t,\ a_{t,j}$ & Frontier size $M_t=|\mathcal C(p_t)|$ and candidate $j$ at frontier $t$. \\
$R,\ Y_{t,a,r},\ Y_{m,r}$ & Evaluator repetitions and the $r$-th observation. Calibration and deployment use fixed $R=4$ observations per scored candidate with balanced presentation order; $Y_{m,r}$ is the fixed-frontier theory notation. \\
$s_p(a),\ s_t(a),\ s_t^*$ & Mean score from the fixed $R=4$ observations on a calibration frontier $p$, the corresponding online mean score at frontier $t$, and $s_t^*=\max_a s_t(a)$. \\
$\mathcal K^{\mathrm{N}}(p_t),\ \mathcal K_t$ & Candidate sets admitted by the native rule and PAC-CF, respectively. \\
$\mathcal P(p)$ & Protected candidates at frontier $p$: actions whose subtrees contain at least one completion accepted by the verifier. \\

\midrule
\multicolumn{2}{l}{\textbf{Fixed-frontier bias-aware analysis}}\\
$m,\ j,\ m^*$ & Candidate/arm indices and the optimal candidate under true downstream value. \\
$\mu_m,\ \mu^*,\ \Delta_m$ & True downstream value, optimum $\mu^*=\max_j\mu_j$, and suboptimality gap $\Delta_m=\mu^*-\mu_m$. \\
$\mathcal F_{r-1},\ \sigma^2$ & History before observation $r$ and conditional sub-Gaussian variance proxy of the centered residual. \\
$L$ & Assumed evaluator-bias bound in Assumption~\ref{ass:persistent_bias}; distinct from the calibrated score-gap threshold $\widehat\tau_Q$. \\
$b_m(n),\ b_m(t)$ & Empirical mean after $n$ evaluator samples; at inspected step $t$, $b_m(t):=b_m(n_m(t))$. \\
$u_{\mathrm{stat}}(n),\ u_{\mathrm{dist}}(n)$ & Statistical radius and bias-robust radius, with $u_{\mathrm{dist}}(n)=u_{\mathrm{stat}}(n)+L$. \\
$\delta,\ \varepsilon$ & Fixed-frontier failure probability and elimination tolerance/slack. \\
$\Delta_{\mathrm{eff},m}$ & Effective certifiable gap $\Delta_m-4L$; proofs may suppress the candidate index and write $\Delta_{\mathrm{eff}}$. \\
$n_m,\ N_m,\ N_m(\tau)$ & Samples allocated to candidate $m$, its elimination sample-complexity bound/scaling, and the realized sample count by proof stopping time $\tau$. \\

\midrule
\multicolumn{2}{l}{\textbf{Conformal calibration on native-controller trajectories}}\\
$\bar Y(a)$ & Aggregated evaluator score for action $a$ on a calibration frontier from the native controller, computed before verifier labeling. \\
$\mathcal E_{\mathrm{N}}(x),\ C_{\mathrm{path}}(p),\ C_{\mathrm{task}}^{\mathrm{N}}(x)$ & Protected calibration frontiers for task $x$, frontier-level path-score gap and task-level nonconformity (the maximum deficit over $\mathcal E_{\mathrm{N}}(x)$). \\
$n,\ C_{(k)}$ & Number of exchangeable calibration tasks with nonempty protected exposure and the $k$-th order statistic of their task scores. \\
$Q,\ k_Q$ & Target conformal coverage and rank $k_Q=\ceil{(n+1)Q}$. \\
$\widehat\tau_Q$ & Frozen operational score gap margin $\widehat\tau_Q=C_{(k_Q)}$; distinct from theoretical bias bound $L$. \\
$X_{\mathrm{new}}$ & Future exchangeable task in the calibration trajectory coverage statement. \\

\midrule
\multicolumn{2}{l}{\textbf{Online PAC-CF deployment}}\\
$\Pi_{\mathrm{N}},\ K$ & Surrounding native controller and its fixed-cardinality top-$K$ parameter when applicable. \\
$\kappa$ & Running request count in provider call equivalents during deployment. \\
$\varepsilon$ & Nonnegative slack in the direct score gap pruning rule. \\
$\tau_{\mathrm{pre}}$ & native ToolTree eligibility gate retained before PAC-CF replaces the subsequent top-$K$ admission step; top-$K$ is not reapplied to PAC-CF survivors. \\

\midrule
\multicolumn{2}{l}{\textbf{Proof-local notation in Appendix~\ref{subsec:theory}}}\\
$q_{m,r},\ \beta_{m,r},\ \xi_{m,r}$ & Conditional evaluator mean, predictable bias $\beta_{m,r}=q_{m,r}-\mu_m$ with $|\beta_{m,r}|\le L$, and centered residual $\xi_{m,r}=Y_{m,r}-q_{m,r}$. \\
$n_{\min}$ & Minimum sample count in the compared pair, $\min(n_m,n_{m^*})$. \\
$C_1,\ C_2,\ W_{-1}$ & Auxiliary constants and the negative Lambert-$W$ branch used in the upper-bound inversion. \\
$\tau,\ \operatorname{kl}(p,q)$ & Stopping time and binary relative entropy used in lower-bound arguments. \\
$\Delta,\ c,\ g$ & Generic two-child true-value gap, proof-local centering constant, and observable separation $g=\Delta-2L$. \\
$\mathcal I,\ \gamma$ & Set of separated alternatives and the arbitrarily small positive perturbation used in the additive lower bound. \\

\midrule
\multicolumn{2}{l}{\textbf{Experimental workload/accounting notation}}\\
$d,\ c,\ \mathcal I_{d,c,B}^{\mathrm{NBC}}$ & Domain and controller indices, and the uncensored paired task subset at request budget $B$. \\
$X,\ R_X^{\mathrm{net}}$ & Generic realized workload quantity (e.g., PReq, Graph, Token) and its net reduction from Native to PAC-CF. \\
$\mathrm{Token}_i,\ p_{ij},\ c_{ij}$ & Exact task level prompt and completion token usage and its prompt and completion counts for each operation. \\
$\mathcal P_i^{\mathrm{M}},\ \phi(u),\ \operatorname{Desc}_{\mathrm{N}}(\phi(u))$ & Roots pruned by PAC-CF that are matched to the saved native controller tree, the native node, and its native descendants. \\
$\mathcal S_i$ & Union of pruned Native roots and their observed Native descendants for task $i$. \\
$V_i^{\mathrm{N}},\ E_i^{\mathrm{N}},\ E_i(\mathcal S_i)$ & Native instantiated nodes, Native scored candidate occurrences, and scored occurrences downstream of the pruned subtrees. \\
$R_{\mathrm{PNode}}^{\mathrm{gross}},\ R_{\mathrm{VNode}}^{\mathrm{gross}}$ & Gross physical node and virtual node reductions reconstructed from Native subtrees. \\
$N_{\mathrm{eff}},\ \Delta U$ & Effective paired task count for a workload cell and utility difference between PAC-CF and Native. \\
$\theta$ & Offline diagnostic score gap threshold in Appendix~\ref{sec:zero_api_audit}; distinct from transition operator $T$. \\

\end{longtable}
\endgroup

\section{Theoretical Boundaries for Biased Exploration}
\label{subsec:theory}

\paragraph{Standing assumption (restated).}
Assumption~\ref{ass:persistent_bias} requires
$|\E[Y_{m,r}\mid\mathcal F_{r-1}]-\mu_m|\le L$ for every active
candidate and evaluation, with a conditionally $\sigma^2$-sub-Gaussian centered
residual. The following statements and proofs use this same fixed-frontier model.

\subsection{Lemma~\ref{lem:concentration}: statement and proof}
\label{app:proof_concentration}

\paragraph{Statement.}
With probability at least $1-\delta$, simultaneously for every
$m\in\mathcal A_t$ and every $n\ge1$,
\[
|b_m(n)-\mu_m|\le u_{\mathrm{stat}}(n)+L=u_{\mathrm{dist}}(n),
\]
where
\[
u_{\mathrm{stat}}(n)=
\sqrt{\frac{2\sigma^2\ln(\pi^2n^2|\mathcal A_t|/(3\delta))}{n}}.
\]
Thus repeated scoring shrinks only the statistical radius; the additive bias
allowance $L$ remains.

\begin{proof}
For evaluation $r$, define the predictable conditional mean
\[
q_{m,r}:=\E[Y_{m,r}\mid\mathcal F_{r-1}],
\qquad
\beta_{m,r}:=q_{m,r}-\mu_m,
\]
and the centered residual
\[
\xi_{m,r}:=Y_{m,r}-q_{m,r}.
\]
Assumption~\ref{ass:persistent_bias} gives $|\beta_{m,r}|\le L$ pathwise, while
$\{\xi_{m,r}\}_{r\ge1}$ is a conditionally $\sigma^2$-sub-Gaussian
martingale-difference sequence. Hence
\begin{equation}
b_m(n)-\mu_m
=
\frac1n\sum_{r=1}^{n}\xi_{m,r}
+
\frac1n\sum_{r=1}^{n}\beta_{m,r}.
\label{eq:pathwise_bias_decomp}
\end{equation}
The second term is deterministically bounded:
\[
\left|
\frac1n\sum_{r=1}^{n}\beta_{m,r}
\right|
\le
\frac1n\sum_{r=1}^{n}|\beta_{m,r}|
\le L.
\]

For the martingale term, conditional sub-Gaussianity implies for every fixed
$m$ and $n$,
\[
\Pr\!\left(
\left|\sum_{r=1}^{n}\xi_{m,r}\right|
\ge n\,u_{\mathrm{stat}}(n)
\right)
\le
2\exp\!\left(
-\frac{n\,u_{\mathrm{stat}}(n)^2}{2\sigma^2}
\right)
=
\frac{6\delta}{\pi^2n^2|\mathcal A_t|}.
\]
Summing over $n\ge1$ and using
$\sum_{n\ge1}n^{-2}=\pi^2/6$ gives failure probability at most
$\delta/|\mathcal A_t|$ for one candidate over all inspected sample counts.
A second union bound over the finite frontier therefore yields, with probability
at least $1-\delta$,
\[
\left|
\frac1n\sum_{r=1}^{n}\xi_{m,r}
\right|
\le u_{\mathrm{stat}}(n)
\qquad
\forall m\in\mathcal A_t,\ \forall n\ge1.
\]
Combining this event with Eq.~\eqref{eq:pathwise_bias_decomp} proves
Eq.~\eqref{eq:main_concentration}. This argument is time-uniform over the
sample counts inspected at the fixed frontier; it does not require the sample
allocation across candidates to be uniform.
\end{proof}

\subsection{Theorem~\ref{thm:complexity}: statement and proof}
\label{app:proof_upper_bound}

\paragraph{Statement.}
On the simultaneous concentration event of Lemma~\ref{lem:concentration}, a
suboptimal candidate $m$ is safely eliminated whenever
\[
2u_{\mathrm{stat}}(n_m)+2u_{\mathrm{stat}}(n_{m^*})
<\Delta_m-4L-\varepsilon.
\]
Hence $\Delta_{\mathrm{eff},m}=\Delta_m-4L$. If
$\Delta_{\mathrm{eff},m}>\varepsilon$, the leading pairwise evaluation complexity is
\[
N_m=O\!\left(
\frac{\sigma^2[\log(|\mathcal A_t|/\delta)+
\log(\Delta_{\mathrm{eff},m}^{-2})]}
{(\Delta_{\mathrm{eff},m}-\varepsilon)^2}
\right).
\]

\begin{proof}
The interval rule eliminates a suboptimal candidate $m$ when
\begin{equation}
b_m(t)+u_{\mathrm{dist}}(n_m)
<
b_{m^*}(t)-u_{\mathrm{dist}}(n_{m^*})-\varepsilon.
\label{eq:appendix_elim_rule}
\end{equation}
Work on the simultaneous event of Lemma~\ref{lem:concentration}. The empirical
means can deviate toward one another by
\[
b_m(t)\le \mu_m+u_{\mathrm{stat}}(n_m)+L,
\qquad
b_{m^*}(t)\ge \mu^*-u_{\mathrm{stat}}(n_{m^*})-L.
\]
Substituting these extremal directions into
Eq.~\eqref{eq:appendix_elim_rule}, a sufficient condition for elimination is
\begin{align*}
(\mu_m+u_{\mathrm{stat}}(n_m)+L)
+(u_{\mathrm{stat}}(n_m)+L)
&<
(\mu^*-u_{\mathrm{stat}}(n_{m^*})-L)\\
&\quad-(u_{\mathrm{stat}}(n_{m^*})+L)-\varepsilon .
\end{align*}
Using $\Delta_m=\mu^*-\mu_m$ gives
\begin{equation}
2u_{\mathrm{stat}}(n_m)+2u_{\mathrm{stat}}(n_{m^*})
<
\Delta_m-4L-\varepsilon,
\label{eq:appendix_upper_condition}
\end{equation}
which is Eq.~\eqref{eq:upper_condition}. This establishes the effective
certifiable gap $\Delta_{\mathrm{eff},m}=\Delta_m-4L$ and proves the safe-elimination
part of the theorem.

The sample counts need not be balanced. Let
$n_{\min}=\min(n_m,n_{m^*})$. The radius
$u_{\mathrm{stat}}(n)=\sqrt{2\sigma^2\ln(C_1n^2)/n}$ is decreasing once $n$ is beyond
a constant-size prefix (equivalently, once $\ln(C_1n^2)>2$), and this prefix is
absorbed by the asymptotic bound below. Hence, in the relevant regime,
\[
u_{\mathrm{stat}}(n_m)\le u_{\mathrm{stat}}(n_{\min}),
\qquad
u_{\mathrm{stat}}(n_{m^*})\le u_{\mathrm{stat}}(n_{\min}),
\]
so the stronger condition
\begin{equation}
4u_{\mathrm{stat}}(n_{\min})
<
\Delta_m-4L-\varepsilon
\label{eq:appendix_nmin_condition}
\end{equation}
is sufficient. In particular, if
$\Delta_m\le4L+\varepsilon$, the right-hand side is non-positive and additional
repeated scoring alone cannot make this certificate fire.

To obtain the leading sample-complexity scaling, set
\[
C_1=\frac{\pi^2|\mathcal A_t|}{3\delta}.
\]
Equation~\eqref{eq:appendix_nmin_condition} requires
\[
4\sqrt{\frac{2\sigma^2\ln(C_1n_{\min}^2)}{n_{\min}}}
<
\Delta_m-4L-\varepsilon.
\]
Squaring and rearranging gives
\begin{equation}
\frac{n_{\min}}{\ln(C_1n_{\min}^2)}
>
\frac{32\sigma^2}{(\Delta_m-4L-\varepsilon)^2}
=:C_2.
\label{eq:appendix_lambert_start}
\end{equation}
The boundary equation can be inverted with the negative branch
$W_{-1}$ of the Lambert $W$ function. Using the standard small-argument expansion
$-W_{-1}(-x)=\log(1/x)+\log\log(1/x)+o(1)$ gives, up to universal constants and
lower-order iterated logarithms,
\[
n_{\min}
=
O\!\left(
\frac{\sigma^2[
\log(|\mathcal A_t|/\delta)+
\log(\Delta_{\mathrm{eff},m}^{-2})]}
{(\Delta_{\mathrm{eff},m}-\varepsilon)^2}
\right).
\]
This is the leading asymptotic scaling reported in
Eq.~\eqref{eq:sample_complexity_main}; it is not claimed as an exact finite-sample
equality \citep{kaufmann2016complexity}. The argument also shows why the
least-sampled member of the compared pair controls the frontier-wise certificate.
\end{proof}

\subsection{Proof of Corollary~\ref{cor:pac_pruning_safety}}
\label{app:proof_pac_pruning_safety}

\begin{proof}
On the simultaneous event of Lemma~\ref{lem:concentration}, elimination implies
\[
\mu_m
\le b_m(n_m)+u_{\mathrm{dist}}(n_m)
<
b_j(n_j)-u_{\mathrm{dist}}(n_j)-\varepsilon
\le \mu_j-\varepsilon
\le \mu^*-\varepsilon .
\]
Thus every eliminated candidate has gap strictly larger than $\varepsilon$.
\end{proof}

\subsection{Theorem~\ref{thm:lower_bound}: statement and proof}
\label{app:proof_lower_boundary}

\begin{theorem}[Structural $2L$ lower boundary]
\label{thm:lower_bound}
Consider the two-child Gaussian subclass of the bounded-bias model with common
variance $\sigma^2$, true gap $\Delta>0$, and a learner required to return the
unique $\varepsilon$-optimal child with probability at least $1-\delta$, where
$\delta<1/2$ and $\Delta>\varepsilon$.
\begin{enumerate}
\item If $\varepsilon<\Delta\le2L$, there exist two admissible instances with
opposite optimal children but identical adaptive observation laws. Uniform
$(\varepsilon,\delta)$-PAC identification is therefore impossible.
\item If $\Delta>2L$, any uniformly $(\varepsilon,\delta)$-PAC learner obeys,
on a suitable admissible pair,
\begin{equation}
\E[\tau]\ge
\frac{2\sigma^2\,\operatorname{kl}(1-\delta,\delta)}
{(\Delta-2L)^2}
=
\Omega\!\left(
\frac{\sigma^2\log(1/\delta)}
{(\Delta-2L)^2}
\right).
\label{eq:two_point_lower_bound_main}
\end{equation}
\end{enumerate}
\end{theorem}

\begin{proof}
For the impossibility statement, consider two children. In world $P$ let the
true-value vector be
\[
(c+\Delta/2,\ c-\Delta/2),
\]
and in world $Q$ swap the two coordinates. If $\Delta\le2L$, choose bounded biases
$(-\Delta/2,+\Delta/2)$ in world $P$ and the swapped biases in world $Q$.
Every observation from either queried child then has mean $c$ in both worlds.
Therefore, under any predictable adaptive sampling policy, the complete
observation-history laws are identical. Yet when $\Delta>\varepsilon$, the sets
of valid $\varepsilon$-optimal outputs are disjoint because the optimal child is
different in the two worlds. No procedure can return the correct
$\varepsilon$-optimal child with probability at least $1-\delta$ in both worlds
for $\delta<1/2$. This proves the first part.

Now assume $\Delta>2L$. Construct two admissible Gaussian worlds with variance
$\sigma^2$ and opposite optimal children, using saturated opposing biases
$\pm L$. The observable mean separation relevant to distinguishing the worlds is
\[
g=\Delta-2L>0.
\]
Let $A$ be the event that the learner returns child 1. Uniform
$(\varepsilon,\delta)$-PAC correctness implies
\[
P(A)\ge1-\delta,
\qquad
Q(A)\le\delta.
\]
By data processing for KL divergence,
\[
\KL(P^H\Vert Q^H)
\ge
\operatorname{kl}(1-\delta,\delta),
\]
where $P^H,Q^H$ denote the full adaptive transcript laws. The adaptive KL chain
rule decomposes the history-level divergence into expected sample counts times
per-sample divergences. For this two-world Gaussian construction, each informative
query contributes at most $g^2/(2\sigma^2)$, hence
\[
\KL(P^H\Vert Q^H)
\le
\E_P[\tau]\frac{g^2}{2\sigma^2}.
\]
Combining the last two displays yields
\[
\E_P[\tau]
\ge
\frac{2\sigma^2\,\operatorname{kl}(1-\delta,\delta)}
{(\Delta-2L)^2},
\]
which is Eq.~\eqref{eq:two_point_lower_bound_main}. Since
$\operatorname{kl}(1-\delta,\delta)=\Omega(\log(1/\delta))$ for
$\delta<1/2$, the stated asymptotic form follows.
\end{proof}

\subsection{Proposition~\ref{prop:additive_lower_bound}: statement and proof}
\label{app:proof_additive_lower_bound}

\begin{proposition}[Additive lower bound for separated alternatives]
\label{prop:additive_lower_bound}
For a finite base frontier, let
\[
\mathcal I=
\{m\ne m^*:\Delta_m>\varepsilon,\ \Delta_m+\varepsilon>2L\}.
\]
Assume that for every $m\in\mathcal I$ the parameter space contains an alternative
that raises candidate $m$ above the original optimum by $\varepsilon+\gamma$ for
arbitrarily small $\gamma>0$. Within the Gaussian bounded-bias subclass, any learner
uniformly $(\varepsilon,\delta)$-PAC over the base instance and these alternatives
must satisfy
\begin{equation}
\E[\tau]\ge
\sum_{m\in\mathcal I}
\frac{2\sigma^2\,\operatorname{kl}(\delta,1-\delta)}
{(\Delta_m+\varepsilon-2L)^2}.
\label{eq:additive_lower_bound_main}
\end{equation}
No finite lower bound from this construction is asserted when
$\Delta_m+\varepsilon\le2L$.
\end{proposition}

\begin{proof}
Fix a base instance and one $m\in\mathcal I$. Construct an alternative instance
that leaves every candidate except $m$ unchanged and raises the true value of
candidate $m$ above the original optimum by $\varepsilon+\gamma$, where
$\gamma>0$ can be arbitrarily small. Assign bias $+L$ to $m$ in the base
instance and bias $-L$ to $m$ in the alternative. The observation-mean shift on candidate $m$ is
\[
\Delta_m+\varepsilon+\gamma-2L,
\]
which is strictly positive by the definition of $\mathcal I$ for sufficiently
small $\gamma$.

Let $A_m$ be the event that the learner returns candidate $m$. Uniform
$(\varepsilon,\delta)$-PAC correctness implies that under the base instance
$P(A_m)\le\delta$, while under the $m$-alternative
$Q_m(A_m)\ge1-\delta$. By data processing and the adaptive KL chain rule,
\[
\E_P[N_m(\tau)]
\frac{(\Delta_m+\varepsilon+\gamma-2L)^2}{2\sigma^2}
\ge
\operatorname{kl}(\delta,1-\delta).
\]
Therefore
\[
\E_P[N_m(\tau)]
\ge
\frac{2\sigma^2\,\operatorname{kl}(\delta,1-\delta)}
{(\Delta_m+\varepsilon+\gamma-2L)^2}.
\]
Letting $\gamma\downarrow0$ gives the per-candidate necessary sample count.
Because each alternative changes only its corresponding candidate, the inequalities
hold simultaneously for every $m\in\mathcal I$ under the same base instance.
Summing them and using
$\tau=\sum_jN_j(\tau)$ yields
\[
\E_P[\tau]
\ge
\sum_{m\in\mathcal I}
\frac{2\sigma^2\,\operatorname{kl}(\delta,1-\delta)}
{(\Delta_m+\varepsilon-2L)^2},
\]
which proves Proposition~\ref{prop:additive_lower_bound}. When
$\Delta_m+\varepsilon\le2L$, this particular separated-alternative construction
does not provide a positive observable shift, so no finite bound from this
construction is asserted.
\end{proof}

\subsection{Summary of Theoretical Regimes}

\begin{table}[H]
\centering
\scriptsize
\caption{Theoretical regimes for a two-child true-value gap $\Delta$. The intermediate
region is not labeled impossible because the lower bound does not rule out
identification there, while the interval certificate does not guarantee deletion.}
\label{tab:theory_regimes}
\setlength{\tabcolsep}{5pt}
\renewcommand{\arraystretch}{0.92}
\begin{tabular}{ll}
\toprule
Gap regime & Consequence \\
\midrule
$\Delta\le\varepsilon$ & Distinguishing the children is unnecessary. \\
$\varepsilon<\Delta\le2L$ & PAC identification can be impossible in the worst case. \\
$\max\{\varepsilon,2L\}<\Delta\le4L+\varepsilon$ &
Identification not ruled out; deletion not certified. \\
$\Delta>4L+\varepsilon$ & Finite reliable elimination is available. \\
\bottomrule
\end{tabular}
\renewcommand{\arraystretch}{1.0}
\end{table}

\subsection{Conditional graceful degradation under severe bias}
\label{sec:appendix_corollary}

An unconditional $4L+\varepsilon$ return quality guarantee does not follow
from Assumption~\ref{ass:persistent_bias} alone: the remaining statistical radius
must also be controlled. The following fixed-frontier corollary states the
corresponding guarantee supported by Lemma~\ref{lem:concentration}.

\begin{corollary}[Graceful degradation for empirically best selection]
\label{cor:degradation}
Work on the concentration event of Lemma~\ref{lem:concentration}, and let
\[
\widehat m\in\argmax_{m\in\mathcal A_t} b_m(n_m)
\]
be the empirically best candidate at a fixed frontier. Then
\begin{equation}
\mu^*-\mu_{\widehat m}
\le
u_{\mathrm{stat}}(n_{m^*})
+
u_{\mathrm{stat}}(n_{\widehat m})
+
2L.
\label{eq:graceful_general}
\end{equation}
Consequently, if
$u_{\mathrm{stat}}(n_{m^*})\le L+\varepsilon/2$ and
$u_{\mathrm{stat}}(n_{\widehat m})\le L+\varepsilon/2$, then
\[
\mu^*-\mu_{\widehat m}\le4L+\varepsilon.
\]
\end{corollary}

\begin{proof}
On the concentration event,
\[
\mu^*
\le
b_{m^*}(n_{m^*})+u_{\mathrm{stat}}(n_{m^*})+L.
\]
Because $\widehat m$ maximizes the empirical mean,
$b_{m^*}(n_{m^*})\le b_{\widehat m}(n_{\widehat m})$, while again by
concentration,
\[
b_{\widehat m}(n_{\widehat m})
\le
\mu_{\widehat m}+u_{\mathrm{stat}}(n_{\widehat m})+L.
\]
Combining the two inequalities gives
Eq.~\eqref{eq:graceful_general}. The stated $4L+\varepsilon$ consequence follows
by substituting the two radius bounds.
\end{proof}

\paragraph{Scope of the severe bias corollary.}
Corollary~\ref{cor:degradation} is deliberately conditional and fixed-frontier.
It does not claim that finite repeated scoring automatically reaches the stated
radius condition, nor does it extend the conformal guarantee on calibration trajectories to
frontiers induced by PAC-CF. It is included only to retain the valid part of the
earlier severe bias intuition without overstating what the current assumptions prove.

\section{Method and Protocol Details}
\label{sec:method_details_appendix}

\subsection{Calibration Construction and Controller Integration}
For calibration, the native controller first runs without PAC-CF and without solution or
future information. Each exposed frontier and its candidates are recorded. Every
scored candidate receives exactly $R=4$ evaluator observations with balanced presentation order,
which are averaged before any calibration statistic is formed. After search terminates,
offline verifier or solution verification information is used to label
\[
\begin{aligned}
\mathcal P(p)=\{a\in\mathcal C(p):\;&\text{the subtree under }a\text{ contains at least one}\\
&\text{solution accepted by the verifier}\}.
\end{aligned}
\]
Tasks with no protected exposure are recorded as missing rather than assigned a
zero protection score. The operational score is computed from the evaluator
scores and labels assigned after search. For Game24, the reconstruction recovers the
protected set using exact symbolic reachability; Blocksworld and GTA use the
verifier or solution verification information specific to each domain retained by the calibration
pipeline. These labels are used only during calibration and are never
available to the online pruning rule. Accordingly, PAC-CF assumes access to an
offline procedure for assigning protection labels appropriate to each domain on calibration traces, not
to protection labels at deployment. The cost and completeness of this offline
procedure are domain dependent; our guarantee is therefore stated with respect to
the protected labels available in each benchmark rather than a universal exact
reachability oracle. The calibration statistic and pruning rule are shared across domains and controllers, while the offline protection labels are constructed using domain-specific verification information. The theoretical bias allowance $L$ in
Assumption~\ref{ass:persistent_bias} is not replaced by or identified with the
empirical $\widehat\tau_Q$.

At test time PAC-CF uses Eq.~\eqref{eq:direct_deployment} only. Candidate
generation, native priority/selection, rollout, backup, termination, and task verification remain with the surrounding controller. For ToolTree, PAC-CF is inserted at the same point as the native pruning rule: the $\tau_{\mathrm{pre}}$ eligibility gate is retained, while the subsequent fixed-cardinality top-$K$ step is replaced by Eq.~\eqref{eq:direct_deployment}. All other ToolTree logic is unchanged. PAC-CF does not evaluate $u_{\mathrm{stat}}$, estimate the theoretical bias bound $L$, construct an online confidence interval, adaptively resample candidates, or take $\delta$ as a deployment input.

\paragraph{Proof of Corollary~\ref{cor:native_trace_path_survival}.}
If $C_{\mathrm{task}}^{\mathrm{N}}(x)\le\widehat\tau_Q$, every exposed native frontier has
$C_{\mathrm{path}}(p)\le\widehat\tau_Q$. Its best protected action therefore has score
deficit at most $\widehat\tau_Q$ and belongs to the keep set in
Eq.~\eqref{eq:direct_deployment}.

\paragraph{Protection Target and Controller Interface.}
The calibration target is existential: on a covered calibration task from the native controller it is enough
that at least one valid continuation survives at each protected frontier, rather
than every viable action. This is a feasibility/path survival guarantee, not a
guarantee that the retained continuation is optimal or maximizes graded utility. For UCT-MCTS and LATS the filter is applied to the
complete scored child set before native rollout selection; for ToolChain* and LevinTS it is
applied before the native priority/Levin ordering; for ToolTree it replaces only
the top-$K$ admission after $\tau_{\mathrm{pre}}$, while subsequent evaluation and backup remain
unchanged. A candidate removed at one frontier occurrence may still reappear if a
different search path generates it again.

\paragraph{Main Text Algorithm.}
Algorithm~\ref{alg:pac_cf_main} gives the complete frontier level PAC-CF procedure.
This appendix therefore records only calibration, controller integration, and
accounting details rather than repeating the pseudocode.

\section{Experiment Protocol Details}
\label{sec:experiment_details_appendix}

\paragraph{Balanced Repeated Scoring.}
Each scored candidate receives exactly four evaluator observations. Candidate order is a
deterministic shuffle determined by the seed with paired forward/reverse ordering within
microbatches. The four scores are aggregated before the frontier decision. This
reduces sensitivity to position effects and stochastic variation without assuming that evaluator bias averages away; the implementation does not adaptively request
additional evaluator samples until a pruning decision becomes certifiable.

\paragraph{Efficiency Metrics.}
The per task budget is a hard cap measured in provider call equivalents. We report physical requests measured in provider call equivalents, Net Graph nodes measuring the realized search structure after adaptive reallocation, and end-to-end prompt and completion tokens. For realized
quantities $X$, reduction is
$1-\sum_i X_i^{\mathrm{PAC}}/\sum_i X_i^{\mathrm{Native}}$ before the declared
controller/domain averaging.

The two gross node metrics are reconstructed from  native ancestry. If
$\mathcal P_i^{\mathrm{M}}$ is the set of roots pruned by PAC-CF whose action paths can be matched to
Native, then
$\mathcal S_i=\bigcup_{u\in\mathcal P_i^{\mathrm{M}}}(\{\phi(u)\}\cup
\operatorname{Desc}_{\mathrm{N}}(\phi(u)))$. Gross physical node reduction is
$\sum_i|\mathcal S_i|/\sum_i|V_i^{\mathrm{N}}|$; overlapping subtrees are unioned once. Gross
virtual node reduction counts scored descendant candidate occurrences in these
subtrees, excluding the directly pruned root because PAC-CF scores it before deciding
to prune. Unmatched roots are not assigned hypothetical descendants. The resulting
metrics therefore quantify observed gross structural suppression and gross downstream
evaluation suppression, respectively. Net graph node reduction instead compares the
complete Native and PAC-CF trajectories after adaptive reallocation and thus quantifies
realized structural benefit.
The saved GTA records do not retain complete parent--child ancestry, so gross
subtree metrics are not reconstructed for GTA.

\paragraph{Calibration Size and Rank Bound.}
All reported controller--domain calibrations use exactly 60 task scores from tasks with protected exposure. Tasks with no protected exposure remain missing, and the prespecified
calibration stream continues until 60 valid task scores are collected. Thus, for
$Q=.95$, $k_{.95}=\ceil{.95(60+1)}=58$, giving the exact rank-based
conformal miscoverage bound $1-58/61=4.92\%$. Technical failures use the same
prespecified completion rule without $Q$ retuning.

\paragraph{Bootstrap Analysis.}
For each condition, utility intervals use 200,000 bootstrap resamples over paired tasks.
Domain averages use 300,000 resamples with task identities shared across the five
controllers; summaries across the three domains and for ToolTree use 300,000 independent
paired resamples within each domain followed by averaging across domains with equal weights. Physical request and
net graph node reductions are ratios of totals within each resample before controller/domain
averaging. Node point estimates are reconstructed directly from the saved search records
and saved counters. This analysis introduces no new LLM calls.

\section{Exact TSP50 Stress Test for Persistent Evaluator Misranking}
\label{sec:tsp50_persistent_misranking}

\paragraph{Exact-Reference Test Bed.}
We generate 24 Euclidean TSP50 instances following the standard random Euclidean
routing setting \citep{kool2019attention} and solve each instance exactly using a
binary MILP with iterative subtour-elimination constraints
\citep{dantzig1954solution} until a single
Hamiltonian cycle remains. A fixed orientation of the optimal tour defines the
optimal reference next city. For each instance, we select four uniformly distributed
control frontiers and four local greedy conflict frontiers before observing any evaluator
outputs. Conflict frontiers are positions at which the globally optimal continuation differs
from the local nearest-neighbor choice. A live Ministral-3-14B evaluator scores the
candidate next cities with $R\in\{1,2,4,8,16\}$ repeated evaluations. The primary
matched cohort includes only frontiers with exactly eight candidates and all 16 evaluator
observations, resulting in 96 control frontiers and 87 conflict frontiers.

For a protected reference continuation $a^*$ with repeated score mean
$\bar Y_R(a^*)$, let $\operatorname{Retain}_{R,K}$ contain all candidates whose mean
score is at least the $K$th largest mean score. We define the catastrophic path
elimination rate as
\[
\mathrm{CPER}_{R,K}
=
\Pr\!\left[
a^*\notin \operatorname{Retain}_{R,K}
\right].
\]

\begin{table*}[t]
\centering
\scriptsize
\caption{Exact TSP50 catastrophic reference path elimination rate (CPER) on the
fixed-cardinality matched cohort. Each row uses the same frontiers for every retained
cardinality $K$; ties at the $K$th score are retained. Lower is better.}
\label{tab:tsp50_cper_rk}
\setlength{\tabcolsep}{5.0pt}
\begin{tabular}{llrrrrr}
\toprule
Stratum & $R$ & $K=1$ & $K=2$ & $K=4$ & $K=6$ & $K=8$ \\
\midrule
Control & 1  & $91.7\%$ & $76.0\%$ & $54.2\%$ & $27.1\%$ & $0.0\%$ \\
        & 2  & $86.5\%$ & $75.0\%$ & $52.1\%$ & $22.9\%$ & $0.0\%$ \\
        & 4  & $87.5\%$ & $79.2\%$ & $53.1\%$ & $22.9\%$ & $0.0\%$ \\
        & 8  & $90.6\%$ & $81.2\%$ & $56.2\%$ & $28.1\%$ & $0.0\%$ \\
        & 16 & $86.5\%$ & $77.1\%$ & $52.1\%$ & $24.0\%$ & $0.0\%$ \\
\midrule
Local-greedy conflict & 1
& $86.2\%$ & $79.3\%$ & $57.5\%$ & $43.7\%$ & $0.0\%$ \\
& 2
& $89.7\%$ & $79.3\%$ & $66.7\%$ & $43.7\%$ & $0.0\%$ \\
& 4
& $92.0\%$ & $81.6\%$ & $62.1\%$ & $41.4\%$ & $0.0\%$ \\
& 8
& $87.4\%$ & $79.3\%$ & $62.1\%$ & $43.7\%$ & $0.0\%$ \\
& 16
& $83.9\%$ & $74.7\%$ & $64.4\%$ & $39.1\%$ & $0.0\%$ \\
\bottomrule
\end{tabular}
\end{table*}

\paragraph{Ranking Persistence under Repeated Evaluation.}
Pruning aggressiveness has a clear effect on catastrophic elimination, whereas additional
evaluator samples do not yield a systematic reduction in CPER. At $R=16$, retaining
four of eight candidates gives CPERs of $52.1\%$ on control frontiers and $64.4\%$
on local greedy conflict frontiers. Even when six of eight candidates are retained,
the corresponding rates remain $24.0\%$ and $39.1\%$. Increasing $R$ from 1 to 16
therefore does not move either stratum into a low CPER regime under fixed-cardinality
pruning.

\paragraph{Local--Global Conflict Risk.}
The conflict stratum isolates frontiers at which the local geometric preference disagrees
with the exact globally optimal continuation. At $R=16$, CPER on conflict frontiers is
$12.3$ percentage points higher than on control frontiers at $K=4$ and $15.1$
percentage points higher at $K=6$. The difference is less informative at $K=1$ and
$K=2$, where elimination rates are already high in both strata. The effect of local--global
disagreement is therefore most visible under moderate irreversible truncation.

\paragraph{Pathwise Bias Accumulation.}
Frontier-level elimination can accumulate across an instance because the exact
reference continuation must survive every inspected frontier.

\begin{table}[H]
\centering
\scriptsize
\caption{Task-level reference path survival at $R=16$. A task survives when
the exact reference continuation is retained at every inspected frontier.}
\label{tab:tsp50_task_survival}
\setlength{\tabcolsep}{5.0pt}
\begin{tabular}{lrrrrr}
\toprule
Stratum & $K=1$ & $K=2$ & $K=4$ & $K=6$ & $K=8$ \\
\midrule
Control
& $0/24$
& $0/24$
& $4/24$ ($16.7\%$)
& $8/24$ ($33.3\%$)
& $24/24$ \\
Conflict
& $0/24$
& $0/24$
& $0/24$
& $1/24$ ($4.2\%$)
& $24/24$ \\
\bottomrule
\end{tabular}
\end{table}

At $K=4$, only $4/24$ control instances preserve the reference path, while no
conflict instance does; at $K=6$, the corresponding counts are $8/24$ and
$1/24$.

\section{Certificate Coverage and Deployment Diagnostics}
\label{sec:zero_api_audit}
\paragraph{Calibration--Deployment Scope.}
All analyses in this section are reconstructed from saved experimental records; no additional
LLM/API calls are made. The calibration trajectory analysis evaluates the object in
Proposition~\ref{prop:native_trace_coverage}, whereas the deployment analysis
examines trajectories generated after PAC-CF changes admission. The latter is an
empirical deployment diagnostic and is not covered by the conformal coverage
statement.

\subsection{Native-Trajectory Coverage and Pruning Aggressiveness}
\paragraph{Frozen-Trajectory Coverage Reconstruction.}
For each reconstructable Game24 controller, sibling candidates and their four saved
evaluator observations are restored from saved native controller records. Exact symbolic
reachability labels the protected set, after which Eq.~\eqref{eq:c_path_compact} is
recomputed task by task. Reconstruction is accepted only when the recovered protected
frontier/action counters agree with the saved oracle diagnostics. ToolTree is reported
as N/A rather than mixing its distinct archived judge channel with the score channel
used by the other controllers.

\paragraph{Controller Coverage Heterogeneity.}
Table~\ref{tab:zero_api_native_coverage} shows that, among reconstructable Game24 controllers, observed task level coverage is $97.00\%$
for UCT-MCTS, $95.88\%$ for ToolChain*, and $100\%$ on the 95 reconstructed LATS
tasks; counting the unresolved LATS task as a failure still gives $98.96\%$. LevinTS
is lower at $93.75\%$ on 80 exposed tasks. This finite realization should therefore
not be summarized as uniformly exceeding $95\%$ across controllers. The reported
Binomial$(80,.95)$ reference probability of approximately $0.371$ for observing at
most 75 successes also indicates that the LevinTS realization is not by itself in
tension with a nominal $.95$ marginal target.

\begin{table}[H]
\centering
\scriptsize
\caption{Game24 native-trajectory coverage and pruning aggressiveness at the frozen
$Q=.95$ operating points. Pruning rates are offline fixed-trajectory diagnostics,
not realized end-to-end savings.}
\label{tab:zero_api_native_coverage}
\setlength{\tabcolsep}{3.2pt}
\begin{tabular}{lrrrr}
\toprule
Controller
& $\widehat\tau_{.95}$
& Task coverage
& Candidate pruning
& Protected-action pruning \\
\midrule
UCT-MCTS
& 30 & $97/100$ ($97.00\%$) & $13.27\%$ & $9.69\%$ \\
ToolChain*
& 25 & $93/97$ ($95.88\%$) & $12.88\%$ & $12.26\%$ \\
LATS
& 30 & $95/95$ ($100\%$)$^\dagger$ & $8.17\%$ & $1.52\%$ \\
LevinTS
& 50 & $75/80$ ($93.75\%$) & $4.61\%$ & $7.27\%$ \\
ToolTree
& 20 & N/A$^\ddagger$ & -- & -- \\
\bottomrule
\end{tabular}

\vspace{1mm}
\raggedright\scriptsize
$^\dagger$One additional exposed LATS task is unresolved; counting it as a failure gives
$95/96=98.96\%$.
$^\ddagger$ToolTree is not reconstructed because its archived judge channel differs from
the score channel used by the other controllers.
\end{table}

\paragraph{Coverage--Aggressiveness Diagnostic.}
The pruning rates are reconstructed at the prespecified frozen margins on native-controller
trajectories and therefore do not measure realized end-to-end savings, since changing
admission can alter subsequent frontiers. LevinTS attains $75/80=93.75\%$ observed
coverage; observing at most 75 successes out of 80 has probability approximately
$0.371$ under an independent Binomial$(80,.95)$ reference and is therefore not by
itself in tension with the nominal $.95$ marginal target. No test-set value is used
to retune the operating point.

\subsection{Oracle Diagnostic on Adaptive Trajectories}
\paragraph{Post-Intervention Path-Survival Diagnostic.}
On the actual PAC-CF trajectory we ask a different question: among controller--task
pairs that expose protected frontiers, did any inspected frontier lose all
continuations accepted by the oracle? This event is downstream of intervention and therefore
is not covered by Proposition~\ref{prop:native_trace_coverage}.

\begin{table}[H]
\centering
\scriptsize
\caption{B100 oracle diagnostic on adaptive trajectories. ``Survival'' means that no
inspected protected frontier loses all actions accepted by the oracle. Game24 ToolTree uses the
canonical $C=20$ replacement run, not the superseded $C=32.5$ run.}
\label{tab:adaptive_oracle_audit}
\setlength{\tabcolsep}{4.2pt}
\begin{tabular}{llrrr}
\toprule
Domain & Controller & Exposed & No loss of all protected actions & Survival \\
\midrule
Game24 & UCT-MCTS & 100 & 97 & $97.00\%$ \\
Game24 & ToolChain* & 97 & 94 & $96.91\%$ \\
Game24 & LATS & 96 & 96 & $100\%$ \\
Game24 & ToolTree & 100 & 99 & $99.00\%$ \\
Game24 & LevinTS & 80 & 75 & $93.75\%$ \\
\cmidrule(lr){2-5}
Game24 & pooled & 473 & 461 & $97.46\%$ \\
Blocksworld & pooled five controllers & 495 & 494 & $99.80\%$ \\
\bottomrule
\end{tabular}
\end{table}

\paragraph{Observed Path Survival.}
As shown in Table~\ref{tab:adaptive_oracle_audit}, at B100, 461 of 473 exposed Game24 controller--task pairs retain at least one
continuation accepted by the oracle at every inspected protected frontier, giving $97.46\%$
pooled survival. Blocksworld records 494 of 495, or $99.80\%$. The Game24
rates for individual controllers range from $93.75\%$ for LevinTS to $100\%$ for LATS. These
figures support the deployment diagnostic but do not extend the calibration trajectory
finite-sample guarantee to frontiers induced by PAC-CF.

\subsection{Sensitivity without ToolTree}
\paragraph{Sensitivity Excluding ToolTree.}
Because ToolTree contributes a large positive utility effect, we remove it from the
macro average across controllers while leaving the other four controllers unchanged. The
latest accounting on the common set of completed pairs still gives positive reductions for all
three realized workload metrics at every budget, whereas the utility shift moves from
slightly positive at B100 to negative at larger budgets.

\begin{table}[H]
\centering
\scriptsize
\caption{Macro average across the three domains excluding ToolTree. Positive reduction means lower
PAC-CF workload.}
\label{tab:exclude_tooltree_macro}
\setlength{\tabcolsep}{2.6pt}
\begin{tabular}{lrrrr}
\toprule
Budget & $\Delta$ utility & P-Req. $\downarrow$ & Net Graph $\downarrow$ & Token $\downarrow$ \\
\midrule
B100 & $+0.265$ & $3.88\%$ & $8.75\%$ & $3.93\%$ \\
B200 & $-0.086$ & $6.41\%$ & $11.86\%$ & $6.32\%$ \\
B300 & $-0.638$ & $6.98\%$ & $12.43\%$ & $7.25\%$ \\
B500 & $-0.855$ & $7.54\%$ & $12.72\%$ & $7.72\%$ \\
\bottomrule
\end{tabular}
\end{table}

\paragraph{ToolTree Ablation Effects.}
Table~\ref{tab:exclude_tooltree_macro} shows that, without ToolTree, physical request reduction rises from $3.88\%$ at B100 to $7.54\%$
at B500, Net Graph reduction from $8.75\%$ to $12.72\%$, and token reduction from
$3.93\%$ to $7.72\%$. Over the same budgets, utility changes from $+0.265$ to
$-0.855$ points. The robustness check therefore isolates two effects: the computation
saving is distributed across controllers, while the positive cross-domain macro-average utility result
is more sensitive to the inclusion of ToolTree.

\subsection{Realized Workload Visualization}
\label{sec:workload_vis_appendix}
\begin{figure*}[H]
\centering
\includegraphics[width=0.91\textwidth]{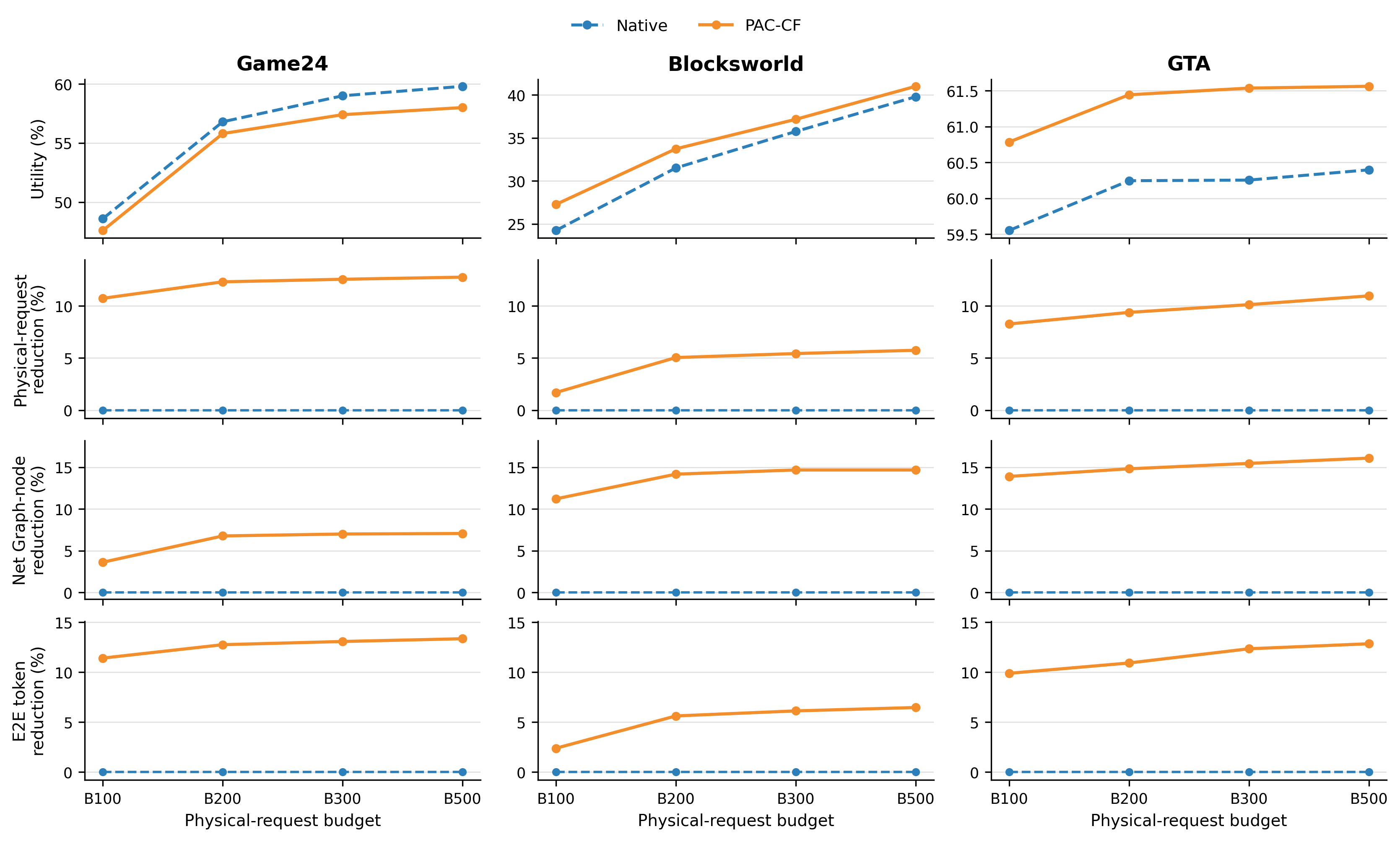}
\caption{Macro average utility and workload reduction for Native and PAC-CF. The top row compares
utility; the remaining rows plot PAC-CF's physical request, Net Graph node, and
end-to-end token reductions relative to Native. Workload points average controllers with equal weights
on the uncensored paired cohort.}
\label{fig:main_budget_efficiency}
\end{figure*}

\subsection{Controlled Gaussian BAI Boundary Validation}
\label{sec:gaussian_bai_boundary}

\paragraph{Empirical Boundary Transition.}
We evaluate the fixed-frontier confidence rule in a controlled Gaussian BAI
environment with opposing bias at the maximum allowed magnitude. The frontier contains $M=20$ candidates
with gap $\Delta=1$, noise level $\sigma=0.2$, confidence level $\delta=0.05$, and
$\varepsilon=0$. The optimal candidate receives constant bias $-L$, while every
suboptimal candidate receives constant bias $+L$. We use the same time-uniform
statistical radius as in the theoretical analysis and predeclare $10{,}000$ Monte
Carlo trials per condition.

\begin{table}[H]
\centering
\scriptsize
\caption{Controlled Gaussian BAI validation around the sufficient separation
boundary. Pruning is the fraction of suboptimal candidates eliminated at the
predeclared $n=5000$ evaluations per candidate.}
\label{tab:safety_ablation_main}
\setlength{\tabcolsep}{4.5pt}
\renewcommand{\arraystretch}{0.92}
\begin{tabular}{ccccc}
\toprule
$L/\Delta$ & $\Delta-4L>0$ & Finite Cert. & Pruning & PAC-CF PCS \\
\midrule
0.05 & Yes      & Yes & 100.0\% & 1.00 \\
0.15 & Yes      & Yes & 100.0\% & 1.00 \\
0.25 & Boundary & No  & 0.0\%   & 1.00 \\
0.40 & No       & No  & 0.0\%   & 1.00 \\
\bottomrule
\end{tabular}
\renewcommand{\arraystretch}{1.0}
\end{table}

\paragraph{Finite-Sample Boundary Shift.}
For balanced sampling, the sufficient finite-sample condition can be written as

$$
4u_{\mathrm{stat}}(n)<\Delta-4L-\varepsilon,
$$

or equivalently,

$$
\frac{L_{\mathrm{crit}}(n)}{\Delta}
=
\frac{1}{4}
-
\frac{u_{\mathrm{stat}}(n)}{\Delta}
-
\frac{\varepsilon}{4\Delta}.
$$

With $\varepsilon=0$, the predicted critical ratios are $0.166$, $0.209$, and
$0.230$ for $n=200$, $1000$, and $5000$, respectively, approaching $0.25$ as
the statistical radius contracts.

\paragraph{Dense-Grid Transition.}
Across the predeclared bias grid, the empirical transition shifts toward the
asymptotic boundary as the number of evaluations increases. At $n=200$, pruning
drops from $94.07\%$ at $L/\Delta=0.20$ to $0.72\%$ at $0.22$ and reaches
$0.00\%$ at $0.23$. At $n=1000$, pruning is $99.998\%$ at $0.22$,
$41.29\%$ at $0.23$, and $0.00\%$ at $0.24$. At $n=5000$, pruning remains
$100.0\%$ at $0.23$, decreases to $56.03\%$ at $0.24$, and reaches
$0.00\%$ at $0.25$. Thus the empirical transition moves toward
$L/\Delta=1/4$ as statistical uncertainty decreases.

\paragraph{Boundary Sufficiency versus Necessity.}
Empirical pruning can occur slightly outside the region certified by the uniform
finite-sample condition. For example, at $n=5000$ and $L/\Delta=0.24$, the
finite-sample certificate is negative while PAC-CF still eliminates $56.03\%$ of
suboptimal candidates. This does not contradict the theory: the condition
$\Delta>4L+\varepsilon$ is sufficient for reliable elimination rather than
necessary for elimination on every realized sample. Across all predeclared
conditions, the optimal candidate is never pruned, and PAC-CF PCS remains $1.00$.

\section{Sensitivity across Models}
\label{sec:cross_model_robustness_appendix}

\paragraph{Secondary-Model Recalibration.}
The secondary model experiment keeps the five controllers, budgets
$B\in\{100,200,300,500\}$, four evaluator observations per candidate, frozen
$Q=.95$, and the same utility and uncensored paired workload accounting as
the primary evaluation, while replacing the evaluator with
\texttt{Ministral-3-14B-Instruct-2512}. Because the score gap distribution is
model dependent, margins are recalibrated and frozen separately rather than
transported from the primary model. Table~\ref{tab:cross_model_q95_tau} gives the
score-gap thresholds used by the reported secondary model runs.

\begin{table}[H]
\centering
\scriptsize
\caption{Frozen secondary model score-gap thresholds $\widehat\tau_{.95}$ on the
common $[0,100]$ score scale.}
\label{tab:cross_model_q95_tau}
\setlength{\tabcolsep}{4.1pt}
\begin{tabular}{lrrrrr}
\toprule
Domain & UCT-MCTS & ToolChain* & LATS & ToolTree & LevinTS \\
\midrule
Game24      & 85.0 & 85.0 & 85.0 & 45.0 & 75.0 \\
Blocksworld & 85.0 & 90.0 & 82.5 & 27.5 & 90.0 \\
GTA         & 75.0 & 77.5 & 52.5 & 47.5 & 30.0 \\
\bottomrule
\end{tabular}
\end{table}

\subsection{Calibration Quantile Sensitivity}
\label{sec:q_sensitivity_14b}

\paragraph{Quantile--Margin Relation.}
To test whether the observed behavior can be reproduced by arbitrarily tightening the
operational pruning margin, we recalibrate the same native-controller trajectories at
$Q=.90$ while retaining $Q=.95$ as the primary operating point.
Table~\ref{tab:cross_model_q90_q95_margin} shows that lowering $Q$ can substantially
reduce the calibrated margin $\widehat\tau_Q$. This sensitivity changes the operational
conformal threshold rather than the theoretical bias allowance $L$.

\begin{table}[H]
\centering
\scriptsize
\caption{Secondary model score-gap thresholds at $Q=.90$ and $Q=.95$.
Each entry reports $\widehat\tau_{.90}/\widehat\tau_{.95}$. Smaller margins permit
more aggressive pruning.}
\label{tab:cross_model_q90_q95_margin}
\setlength{\tabcolsep}{3.7pt}
\begin{tabular}{lccccc}
\toprule
Domain & UCT-MCTS & ToolChain* & LATS & ToolTree & LevinTS \\
\midrule
Game24      & 80/85 & 80/85 & 77.5/85 & 32.5/45 & 65/75 \\
Blocksworld & 70/85 & 15/90 & 10/82.5 & 0/27.5 & 65/90 \\
GTA         & 50/75 & 35/77.5 & 20/52.5 & 42.5/47.5 & 25/30 \\
\bottomrule
\end{tabular}
\end{table}

\paragraph{Quantile Utility--Efficiency Tradeoff.}
The smaller $Q=.90$ margins induce more aggressive pruning but do not uniformly
preserve utility. Table~\ref{tab:cross_model_q90_sensitivity} reports the endpoint
budgets for Game24 and Blocksworld.

\begin{table}[H]
\centering
\scriptsize
\caption{Endpoint deployment sensitivity under the more aggressive $Q=.90$
calibration. Positive workload values denote reductions relative to Native search.}
\label{tab:cross_model_q90_sensitivity}
\setlength{\tabcolsep}{3.3pt}
\begin{tabular}{llrrrr}
\toprule
Domain & Budget & $\Delta U$ & P-Req. $\downarrow$ & Net Graph $\downarrow$ &
Token $\downarrow$ \\
\midrule
Game24
& B100 & $+3.41$ pp & 7.90\% & 17.36\% & 8.59\% \\
& B500 & $+0.41$ pp & 16.10\% & 22.82\% & 16.93\% \\
\midrule
Blocksworld
& B100 & $+0.00$ pp & 16.54\% & 34.18\% & 17.37\% \\
& B500 & $-7.80$ pp & 22.33\% & 37.14\% & 23.40\% \\
\bottomrule
\end{tabular}
\end{table}

The more aggressive operating point increases workload reduction but can substantially
degrade utility, as seen for Blocksworld at B500. Thus the utility--efficiency behavior
depends materially on the calibrated operating point rather than on an arbitrary
score-gap threshold. GTA $Q=.90$ margins are reported in
Table~\ref{tab:cross_model_q90_q95_margin}, but a complete matched end-to-end
$Q=.90$ GTA evaluation is unavailable.

\paragraph{Cross-Model Utility--Efficiency Contrast.}
Table~\ref{tab:cross_model_domain_all_budgets} reports the complete secondary model
budget sweep. Game24 occupies the strongest region: utility improves by $+0.80$ and
$+0.40$ points at B100 and B200, then changes by $-1.00$ and $-0.40$ points at B300
and B500, while physical request, Net Graph, and token reductions remain between
$13.37\%$--$15.09\%$, $16.52\%$--$17.84\%$, and $14.07\%$--$15.77\%$ across the
four budgets. GTA shows a different frontier: utility changes remain within
$-0.15$ to $-0.59$ points while the three realized reductions increase with budget,
reaching $9.81\%$, $13.34\%$, and $11.97\%$ at B500. Blocksworld is the least
favorable domain for the secondary model: utility changes by $-0.40$ points at each
budget, physical requests are nearly unchanged, and the principal effect appears in
search structure, with Net Graph reduction decreasing from $3.81\%$ at B100 to
$2.75\%$ at B500. The resulting macro average across the three domains retains positive workload reductions at
every budget but moves from a $+0.08$-point utility change at B100 to $-0.38$ points
at B500.

\begin{table*}[ht]
\centering
\scriptsize
\caption{Secondary model results at the frozen $Q=.95$ operating point across all budgets. Utility changes
use the full paired test set; workload reductions use uncensored paired
rows.}
\label{tab:cross_model_domain_all_budgets}
\setlength{\tabcolsep}{3.5pt}
\begin{tabular}{llrrrr}
\toprule
Domain & Budget & $\Delta U$ & P-Req. $\downarrow$ & Net Graph $\downarrow$ & E2E Token $\downarrow$ \\
\midrule
Game24 & B100 & $+0.80$ pp & 13.37\% & 17.84\% & 14.07\% \\
       & B200 & $+0.40$ pp & 13.63\% & 16.52\% & 14.31\% \\
       & B300 & $-1.00$ pp & 15.09\% & 17.06\% & 15.77\% \\
       & B500 & $-0.40$ pp & 13.76\% & 16.72\% & 15.37\% \\
\midrule
Blocksworld & B100 & $-0.40$ pp & $-0.17\%$ & 3.81\% & 0.39\% \\
            & B200 & $-0.40$ pp & $-0.26\%$ & 3.18\% & 0.29\% \\
            & B300 & $-0.40$ pp & $-0.29\%$ & 2.83\% & 0.27\% \\
            & B500 & $-0.40$ pp & $-0.30\%$ & 2.75\% & 0.26\% \\
\midrule
GTA & B100 & $-0.15$ pp & 5.47\% & 9.49\% & 5.93\% \\
    & B200 & $-0.59$ pp & 6.60\% & 10.63\% & 7.70\% \\
    & B300 & $-0.46$ pp & 8.79\% & 12.41\% & 11.40\% \\
    & B500 & $-0.34$ pp & 9.81\% & 13.34\% & 11.97\% \\
\midrule
Macro average & B100 & $+0.08$ pp & 6.22\% & 10.38\% & 6.80\% \\
               & B200 & $-0.20$ pp & 6.66\% & 10.11\% & 7.43\% \\
               & B300 & $-0.62$ pp & 7.86\% & 10.77\% & 9.15\% \\
               & B500 & $-0.38$ pp & 7.76\% & 10.94\% & 9.20\% \\
\bottomrule
\end{tabular}
\end{table*}

\paragraph{Matched Calibration Size.}
For the secondary model robustness study, each controller--domain margin is calibrated
from 60 tasks with protected exposure at the prespecified $Q=.95$, giving $k_{.95}=58$.
The margin is frozen before deployment and is not selected from test outcomes. Gross
subtree metrics remain unavailable for GTA because complete ancestry is not retained;
realized request, Net Graph, and token accounting are unaffected.

\section{Frozen Configuration and Evaluation Scope}
\label{sec:frozen_config_appendix}

\paragraph{Frozen Configuration.}
Table~\ref{tab:frozen_config} summarizes the implementation, calibration target, evaluator protocol, budgets, and paired analysis sets fixed before the matched Native/PAC-CF comparisons.

\begin{table}[H]
\centering
\scriptsize
\caption{Frozen implementation and evaluation configuration. These settings are
fixed across the matched Native/PAC-CF comparisons unless a row explicitly
states otherwise.}
\label{tab:frozen_config}
\setlength{\tabcolsep}{5pt}
\begin{tabular}{p{0.27\linewidth}p{0.64\linewidth}}
\toprule
Item & Frozen setting \\
\midrule
Controllers & UCT-MCTS, ToolChain*, LATS, ToolTree, LevinTS \\
Provider model & \texttt{deepseek-v4-flash}; seed 42 \\
Budgets & $B\in\{100,200,300,500\}$, measured in provider call equivalents per task \\
Main calibration & Direct conformal calibration for each domain--controller pair with $n=60$ tasks with protected exposure per cell, $Q=0.95$; deployed $\varepsilon=0$ \\
Frontier scoring & Four evaluator observations per candidate; matched evaluator temperature $0.2$; deterministic shuffle set by the seed with paired forward/reverse order \\
Calibration labels & Exact verifier / solution verification information assigned after search; never exposed to online search \\
Online intervention & Eq.~\eqref{eq:direct_deployment} only; native generator, downstream priority, rollout/backup, termination, and verifier unchanged \\
Efficiency accounting & Three primary metrics: physical requests, net graph nodes, and end-to-end tokens; gross physical and virtual nodes are supplementary; workload metrics use all paired tasks as the primary capped-deployment estimand; uncensored paired rows provide a natural-termination sensitivity analysis \\
Primary paired analysis sets & Utility and primary workload:
Game24 $N=100$, Blocksworld/GTA $N=99$; natural-termination sensitivity: uncensored paired subsets within each domain--controller--budget cell. \\
\bottomrule
\end{tabular}
\end{table}

\paragraph{Trace-Compatible Identifiers.}
The released CSVs retain the internal controller identifiers
\texttt{vanilla}, \texttt{astar}, and \texttt{lts} for trace compatibility;
they correspond to UCT-MCTS, ToolChain*, and LevinTS, respectively. The identifiers
\texttt{lats} and \texttt{tooltree} already match their public names.

\paragraph{Three-Domain Macro.}
The primary macro evaluation uses Game24, Blocksworld, and GTA.

\paragraph{Calibration Rank Summary.}
Every controller--domain calibration uses $n=60$ task scores with protected exposure.
Tasks without protected exposure are treated as missing and the prespecified stream
continues until 60 valid scores are collected. Hence $k_{.95}=58$ and the exact
rank-based miscoverage bound is $1-58/61=4.92\%$ for every cell.
Table~\ref{tab:calibration_rank_audit} reports only the quantity that varies across
cells, the frozen score-gap threshold $\widehat\tau_{.95}$.

\begin{table}[H]
\centering
\scriptsize
\caption{Frozen $Q=.95$ score-gap thresholds $\widehat\tau_{.95}$.
All cells use $n=60$, $k_{.95}=58$, and the same $4.92\%$ rank-based
miscoverage bound.}
\label{tab:calibration_rank_audit}
\setlength{\tabcolsep}{4.0pt}
\begin{tabular}{lrrrrr}
\toprule
Domain & UCT-MCTS & ToolChain* & LATS & ToolTree & LevinTS \\
\midrule
Game24
& 30.0 & 25.0 & 30.0 & 20.0 & 50.0 \\
Blocksworld
& 45.0 & 50.0 & 62.5 & 0.0 & 10.0 \\
GTA
& 40.0 & 40.0 & 40.0 & 30.0 & 30.0 \\
\bottomrule
\end{tabular}
\end{table}

The zero Blocksworld ToolTree margin means that its best protected action never fell
below the frontier maximum on calibration tasks with protected exposure; it does not
imply that other controller--domain pairs require no protection.

\section{Results by Controller}
\label{sec:per_backbone_main}

\paragraph{Controller Heterogeneity.}
Table~\ref{tab:per_backbone_main} reports every controller included in the macro averages in
Table~\ref{tab:macro_main}. Utility deltas are PAC-CF minus Native, and positive
reductions mean lower PAC-CF workload.
\begin{table*}[ht]
\centering
\scriptsize
\caption{Results by controller at the frozen $Q=.95$ operating point for B100 and B500. Utility deltas use the full test set; $N_{\mathrm{eff}}$ and workload reductions use uncensored paired rows.}
\label{tab:per_backbone_main}
\setlength{\tabcolsep}{1.2pt}
\resizebox{\textwidth}{!}{%
\begin{tabular}{llrrrrrrrrrr}
\toprule
Domain & Controller & \multicolumn{5}{c}{B100} & \multicolumn{5}{c}{B500} \\
\cmidrule(lr){3-7}\cmidrule(lr){8-12}
& & $\Delta U$ & $N_{\mathrm{eff}}$ & PReq & NetG & Token & $\Delta U$ & $N_{\mathrm{eff}}$ & PReq & NetG & Token \\
\midrule
Game24 & UCT-MCTS & +1.00 & 75 & +2.07\% & +11.64\% & +1.81\% & +0.00 & 88 & +2.41\% & +15.72\% & +2.24\% \\
 & ToolChain* & +0.00 & 57 & +4.80\% & +16.21\% & +2.53\% & -2.00 & 100 & +8.69\% & +19.73\% & +6.38\% \\
 & LATS & +0.00 & 68 & +0.00\% & +4.49\% & +0.00\% & +0.00 & 100 & +4.81\% & +15.36\% & +4.77\% \\
 & ToolTree & -3.00 & 99 & +46.91\% & -19.13\% & +53.11\% & -3.00 & 100 & +46.89\% & -21.34\% & +52.53\% \\
 & LevinTS & -3.00 & 61 & -0.12\% & +4.95\% & -0.38\% & -4.00 & 100 & +1.01\% & +5.85\% & +0.85\% \\
\midrule
Blocksworld & UCT-MCTS & +1.01 & 30 & +2.26\% & +8.53\% & +2.29\% & +3.03 & 47 & +5.16\% & +8.67\% & +4.87\% \\
 & ToolChain* & +2.02 & 69 & +4.14\% & +3.68\% & +4.01\% & +0.00 & 99 & +3.59\% & +3.11\% & +3.34\% \\
 & LATS & +0.00 & 28 & +0.00\% & +1.59\% & +0.00\% & -1.01 & 57 & +0.00\% & +0.31\% & +0.00\% \\
 & ToolTree & +9.09 & 99 & -4.30\% & +31.53\% & +0.04\% & +9.09 & 99 & -4.30\% & +31.53\% & +0.04\% \\
 & LevinTS & +3.03 & 61 & +6.34\% & +10.76\% & +5.57\% & -5.05 & 97 & +24.26\% & +29.74\% & +24.06\% \\
\midrule
GTA & UCT-MCTS & +0.21 & 73 & +22.55\% & +34.27\% & +25.91\% & +0.03 & 93 & +33.85\% & +41.67\% & +38.82\% \\
 & ToolChain* & +0.00 & 85 & +0.94\% & +1.69\% & +1.46\% & +0.00 & 97 & +0.92\% & +1.59\% & +1.23\% \\
 & LATS & -0.32 & 68 & +0.00\% & +0.00\% & +0.00\% & -0.31 & 97 & +1.34\% & +3.38\% & +1.47\% \\
 & ToolTree & +7.05 & 99 & +14.24\% & +26.29\% & +18.14\% & +7.05 & 99 & +14.24\% & +26.29\% & +18.14\% \\
 & LevinTS & -0.77 & 86 & +3.60\% & +7.24\% & +3.93\% & -0.95 & 99 & +4.46\% & +7.52\% & +4.59\% \\
\bottomrule
\end{tabular}}
\end{table*}

\paragraph{Utility--Efficiency Heterogeneity.}
At B500, GTA UCT-MCTS reduces physical requests by $33.85\%$, Net Graph nodes by
$41.67\%$, and tokens by $38.82\%$ with a $+0.03$ utility shift, whereas Game24
ToolTree reduces requests by $46.89\%$ and tokens by $52.53\%$ but loses three
success points and increases Net Graph nodes by $21.34\%$. Blocksworld ToolTree
shows yet another regime: utility improves by $9.09$ points and Net Graph nodes fall by
$31.53\%$, while physical requests increase by $4.30\%$ and token use is nearly
unchanged. No single workload metric therefore captures all deployment effects.

\section{Budget Censoring Sensitivity and Gross Subtree Accounting}
\label{sec:budget_censoring_sensitivity}

\paragraph{Budget-Censoring Sensitivity.}
Exhausting the hard budget truncates the observed trajectory at the experiment cap rather
than at its natural termination. Our primary capped-deployment analysis includes all paired tasks and therefore
does not condition on an event that may be affected by the intervention. We additionally report the uncensored paired cohort as a descriptive natural-termination sensitivity analysis. Because cohort inclusion requires both runs to terminate before the cap and may therefore be affected by the intervention, we do not assign the uncensored-cohort comparison a causal interpretation. As shown in Table~\ref{tab:budget_censor_sensitivity}, across B100--B500, the reductions in the analysis over all tasks
are $6.14$--$9.47\%$ for physical requests, $9.03$--$12.40\%$ for Net
Graph nodes, and $6.84$--$10.25\%$ for tokens; the corresponding uncensored paired
values are uniformly larger but keep the same positive sign.

\begin{table}[!t]
\centering
\scriptsize
\caption{B500 matched pruning and realized graph reduction on the uncensored paired cohort.
Physical Nodes is gross matched native subtree suppression; Net Graph is realized
end-to-end reduction after adaptive reallocation. ``Matched roots'' is the fraction
of PAC-CF pruning roots whose Native ancestry is recoverable. GTA is omitted because
complete node ancestry was not recorded.}
\label{tab:gross_subtree_main}
\setlength{\tabcolsep}{3.8pt}
\begin{tabular}{cccccc}
\toprule
Domain & Controller & $N_{\mathrm{eff}}$ & Physical Nodes $\downarrow$ & Net Graph $\downarrow$ & Matched roots \\
\midrule
Game24 & UCT-MCTS & 88 & 15.79\% & 15.72\% & 99.66\% \\
 & ToolChain* & 100 & 21.61\% & 19.73\% & 98.96\% \\
 & LATS & 100 & 15.34\% & 15.36\% & 100.00\% \\
 & ToolTree & 100 & 11.42\% & $-21.34\%$ & 81.44\% \\
 & LevinTS & 100 & 7.60\% & 5.85\% & 95.51\% \\
 & Macro & -- & 14.35\% & 7.06\% & 95.11\% \\
\midrule
Blocksworld & UCT-MCTS & 47 & 11.91\% & 8.67\% & 89.47\% \\
 & ToolChain* & 99 & 6.95\% & 3.11\% & 100.00\% \\
 & LATS & 57 & 0.31\% & 0.31\% & 100.00\% \\
 & ToolTree & 99 & 57.18\% & 31.53\% & 56.25\% \\
 & LevinTS & 97 & 30.58\% & 29.74\% & 100.00\% \\
 & Macro & -- & 21.38\% & 14.67\% & 89.14\% \\
\bottomrule
\end{tabular}
\end{table}

\section{Paired Statistical Analysis}
\label{sec:statistical_audit}

\paragraph{Paired Bootstrap.}
All statistics in this section are offline analyses of the frozen task level
artifacts; they introduce no new LLM or API calls. The primary analysis uses
Game24 $N=100$ paired tasks and Blocksworld/GTA $N=99$ paired tasks.
Utility intervals for individual conditions use
200,000 nonparametric bootstrap resamples over paired tasks. Domain level intervals
use 300,000 resamples with task identities shared across the five controllers
before averaging controllers with equal weights. Summaries across the three domains and for ToolTree
use 300,000 independent paired resamples within each domain followed by
averaging across domains with equal weights. Workload reductions are ratios of totals over uncensored paired rows inside
each resample before controller/domain averaging; end-to-end token usage is resampled in
the same way. Utility remains on the full task set. Rescue/regression/tie (R/G/T) counts compare PAC-CF and Native
task level utility, with exact ties defined at numerical tolerance $10^{-12}$.

\begin{table*}[ht]
\centering
\scriptsize
\caption{Paired bootstrap averages by domain. Utility intervals use the full paired test set; workload intervals use uncensored paired rows. Positive reduction means lower PAC-CF workload.}
\label{tab:domain_macro_ci}
\setlength{\tabcolsep}{1.6pt}
\resizebox{\textwidth}{!}{%
\begin{tabular}{llcccc}
\toprule
Domain & Budget & $\Delta$ utility [95\% CI] & PReq [95\% CI] & NetG [95\% CI] & Token [95\% CI] \\
\midrule
Game24 & B100 & $-1.00$ [-2.60,0.60] & $+10.73\%$ [8.92,12.34] & $+3.63\%$ [0.40,6.50] & $+11.41\%$ [9.66,12.95] \\
Game24 & B200 & $-1.00$ [-2.60,0.40] & $+12.31\%$ [10.58,13.90] & $+6.78\%$ [3.26,9.87] & $+12.76\%$ [11.08,14.27] \\
Game24 & B300 & $-1.60$ [-3.00,-0.40] & $+12.56\%$ [10.98,14.01] & $+7.01\%$ [3.71,9.91] & $+13.08\%$ [11.50,14.50] \\
Game24 & B500 & $-1.80$ [-3.20,-0.60] & $+12.76\%$ [11.19,14.22] & $+7.06\%$ [3.80,9.95] & $+13.35\%$ [11.76,14.81] \\
\midrule
Blocksworld & B100 & $+3.03$ [1.41,4.85] & $+1.69\%$ [-0.73,4.11] & $+11.22\%$ [9.06,13.46] & $+2.38\%$ [0.07,4.70] \\
Blocksworld & B200 & $+2.22$ [0.81,3.64] & $+5.04\%$ [2.06,7.97] & $+14.17\%$ [11.58,16.70] & $+5.62\%$ [2.72,8.45] \\
Blocksworld & B300 & $+1.41$ [-0.20,3.03] & $+5.43\%$ [2.70,8.20] & $+14.67\%$ [12.28,17.13] & $+6.12\%$ [3.46,8.83] \\
Blocksworld & B500 & $+1.21$ [-0.61,3.23] & $+5.74\%$ [2.92,8.59] & $+14.67\%$ [12.25,17.17] & $+6.46\%$ [3.71,9.24] \\
\midrule
GTA & B100 & $+1.23$ [0.33,2.24] & $+8.27\%$ [6.75,9.75] & $+13.90\%$ [11.74,16.02] & $+9.89\%$ [8.04,11.67] \\
GTA & B200 & $+1.20$ [0.29,2.20] & $+9.38\%$ [7.63,11.05] & $+14.81\%$ [12.49,17.07] & $+10.92\%$ [8.84,12.89] \\
GTA & B300 & $+1.28$ [0.37,2.29] & $+10.12\%$ [8.22,11.97] & $+15.46\%$ [13.05,17.81] & $+12.35\%$ [9.61,14.95] \\
GTA & B500 & $+1.17$ [0.26,2.17] & $+10.96\%$ [8.58,13.23] & $+16.09\%$ [13.36,18.74] & $+12.85\%$ [9.98,15.45] \\
\bottomrule
\end{tabular}}
\end{table*}

\paragraph{Domain-Level Utility--Efficiency Heterogeneity.}
GTA has positive utility intervals at all four budgets, from $+1.23$
$[0.33,2.24]$ at B100 to $+1.17$ $[0.26,2.17]$ at B500, together with positive
intervals for all workload metrics. Blocksworld is resolved positive in utility at
B100 and B200 but becomes unresolved at B300 and B500. Game24 instead has resolved
negative utility at B300 and B500, with intervals $[-3.00,-0.40]$ and
$[-3.20,-0.60]$, even though all three workload intervals remain positive. The
aggregate result should therefore be interpreted as an observed efficiency pattern with a
utility tradeoff that varies by domain.

\begin{table}[H]
\centering
\scriptsize
\caption{Independent summary across the three domains for the already pruning-aware ToolTree controller. Utility uses the full test set; workload intervals use uncensored paired rows.}
\label{tab:strong_backbone_ci}
\setlength{\tabcolsep}{1.8pt}
\resizebox{\linewidth}{!}{%
\begin{tabular}{lcccc}
\toprule
Budget & $\Delta$ utility [95\% CI] & PReq [95\% CI] & NetG [95\% CI] & Token [95\% CI] \\
\midrule
B100 & $+4.38$ [1.43,7.41] & $+18.95\%$ [14.81,22.96] & $+12.90\%$ [7.54,17.77] & $+23.76\%$ [19.56,27.85] \\
B200 & $+4.38$ [1.43,7.41] & $+18.94\%$ [14.84,22.93] & $+12.16\%$ [6.62,17.20] & $+23.57\%$ [19.39,27.65] \\
B300 & $+4.38$ [1.43,7.40] & $+18.94\%$ [14.84,22.93] & $+12.16\%$ [6.67,17.19] & $+23.57\%$ [19.39,27.65] \\
B500 & $+4.38$ [1.44,7.41] & $+18.94\%$ [14.83,22.94] & $+12.16\%$ [6.66,17.20] & $+23.57\%$ [19.40,27.65] \\
\bottomrule
\end{tabular}}
\end{table}

\paragraph{ToolTree Tradeoff by Budget.}
Table~\ref{tab:strong_backbone_ci} shows that, for the already pruning-aware ToolTree controller, the average utility gain across domains remains
$+4.38$ points with a strictly positive 95\% interval at every budget. Physical request
reduction stays at about $18.94$--$18.95\%$, Net Graph reduction at
$12.16$--$12.90\%$, and token reduction at $23.57$--$23.76\%$. This aggregate pattern does
not imply that every ToolTree domain improves: the table by controller shows the
negative Game24 utility case and the positive Blocksworld/GTA cases that combine into
the average across domains.

\end{document}